\documentclass[letterpaper]{article} 
\usepackage[draft]{aaai25}  
\usepackage{times}  
\usepackage{helvet}  
\usepackage{courier}  
\usepackage[hyphens]{url}  
\usepackage{graphicx} 
\usepackage{natbib}  
\usepackage{caption} 
\usepackage{algorithm}
\usepackage{algorithmic}

\usepackage{amsmath}
\usepackage{amssymb}
\usepackage{mathtools}
\usepackage{amsthm}

\usepackage[utf8]{inputenc} 
\usepackage{url}            
\usepackage{booktabs}       
\usepackage{amsfonts}       
\usepackage{nicefrac}       
\usepackage{microtype}      
\usepackage{xcolor}         
\usepackage{amsmath,amsthm}        
\usepackage{subcaption}      
\usepackage{pifont}
\usepackage[capitalize,noabbrev]{cleveref}
\usepackage{diagbox}
\usepackage{makecell}
\usepackage{multirow}
\usepackage{comment}

\usepackage[capitalize,noabbrev]{cleveref}

\theoremstyle{plain}
\newtheorem{theorem}{Theorem}[section]

\newtheorem{lemma}[theorem]{Lemma}

\theoremstyle{definition}
\newtheorem{definition}[theorem]{Definition}
\newtheorem{assumption}[theorem]{Assumption}
\theoremstyle{remark}

\DeclareMathOperator{\kl}{D_{\mathrm{KL}}}
\DeclareMathOperator{\DO}{\boldsymbol{do}}
\DeclareMathOperator{\Ent}{\mathcal{H}}
\DeclareMathOperator{\MI}{\mathcal{I}}
\DeclareMathOperator{\E}{\mathbb{E}}

\newcommand{\taujt}{\boldsymbol{\tau}}
\newcommand{\rolejt}{\mathbf{c}}
\newcommand{\rolei}{\mathbf{c}_i}
\newcommand{\queryjt}{\mathbf{q}}
\newcommand{\queryi}{\mathbf{q}_i}
\newcommand{\influjt}{\boldsymbol{\bar{I}}}
\newcommand{\influi}{\boldsymbol{\bar{I}}_i}
\newcommand{\counterinflui}{\boldsymbol{\widetilde{I}}_i}
\newcommand{\priorjt}{\mathcal{P(\mathbf{c})}}
\newcommand{\postjt}{\mathcal{P}(\rolejt|\influjt,\queryjt)}
\newcommand{\margjt}{\mathcal{P}(\rolejt|\queryjt)}
\newcommand{\posti}{\mathcal{P}(\rolei|\influi, \queryi)}
\newcommand{\margi}{\mathcal{P}(\rolei|\queryi)}
\newcommand{\margdoi}{\mathcal{P}(\rolei|\DO(\influi), \queryi)}

\usepackage{newfloat}
\usepackage{listings}
\DeclareCaptionStyle{ruled}{labelfont=normalfont,labelsep=colon,strut=off} 
\floatstyle{ruled}
\newfloat{listing}{tb}{lst}{}
\floatname{listing}{Listing}
\title{CORD: Generalizable Cooperation via Role Diversity}
\author {
    Kanefumi Matsuyama\textsuperscript{\rm 1}, 
    Kefan Su\textsuperscript{\rm 1}, 
    Jiangxing Wang\textsuperscript{\rm 1}, 
    Deheng Ye\textsuperscript{\rm 2}, 
    Zongqing Lu\textsuperscript{\rm 1}
}
\affiliations {
    \textsuperscript{\rm 1}School of Computer Science, Peking University\\
    \textsuperscript{\rm 2}Tencent Inc.\\
    \{matsuyama@,\,sukefan@,\,jiangxiw@stu.,\,zongqing.lu@\}pku.edu.cn, dericye@tencent.com
}

\usepackage{bibentry}

\begin{document}

\maketitle

\begin{abstract}
Cooperative multi-agent reinforcement learning (MARL) aims to develop agents that can collaborate effectively. However, most cooperative MARL methods overfit training agents, making learned policies not generalize well to unseen collaborators, which is a critical issue for real-world deployment. Some methods attempt to address the generalization problem but require prior knowledge or predefined policies of new teammates, limiting real-world applications. To this end, we propose a hierarchical MARL approach to enable generalizable cooperation via role diversity, namely CORD. CORD's high-level controller assigns roles to low-level agents by maximizing the role entropy with constraints. We show this constrained objective can be decomposed into causal influence in role that enables reasonable role assignment, and role heterogeneity that yields coherent, non-redundant role clusters. Evaluated on a variety of cooperative multi-agent tasks, CORD achieves better performance than baselines, especially in generalization tests. Ablation studies further demonstrate the efficacy of the constrained objective in generalizable cooperation.
\end{abstract}

%

\section{Introduction}
\label{section: intro}
Cooperative multi-agent reinforcement learning (MARL), where agents cooperate to maximize the shared reward, has a broad range of applications, from autonomous warehouse \citep{zhou2021multi}, power dispatch \citep{wang2021multiagent} to logistics \citep{li2019cooperative}, inventory management \citep{feng2022multiagent}. Recent years have witnessed substantial progress in different kinds of cooperative MARL algorithms, including value decomposition \citep{sunehag2018value,qmix,son2019qtran,wang2020qplex,refil}, multi-agent actor-critic \citep{yu2022surprising,zhang2021fop,kubatrust,wang2023more}, and fully decentralized learning \citep{su2022ma2ql,jiang2022i2q,su2022dpo}.  

However, existing algorithms often result in policies that overfit the co-trained teammates and scenarios encountered during training, thus lacking the generalization ability to cooperate effectively with new teammates in new scenarios. As agents may fail to adapt to unforeseen partners, the performance of the multi-agent system can degrade dramatically or even collapse entirely. 
To address the generalization challenge, a few methods \citep{barrett2015cooperating,gu2021online,copa} have been proposed. However, most approaches require prior knowledge or predefined policies about new teammates, limiting their applicability in complex real-world environments. In addition, cooperative MARL tasks usually require role divisions. Thus, agents may need to play different roles to collaborate with unseen teammates. However, existing work lacks the essential ability to learn the role assignment. Either not considering role division, or only learning role heterogeneity while neglecting complex inter-agent interactions, making it difficult to generalize in environments that require various roles.


In this paper, we propose a hierarchical MARL approach to enable generalizable \textbf{CO}operation via \textbf{R}ole \textbf{D}iversity, namely \textbf{CORD}. CORD does not depend on pre-defined agent policies or behaviors and can be trained end-to-end. In CORD, a high-level controller is responsible for analyzing the environment and assigning roles to low-level agents. The low-level agents then condition their policies on the assigned roles. 
To enable generalizable role assignment when collaborating with unseen teammates, we maximize the entropy of the role distribution with a certain constraint. This constraint is formulated as a causal relationship between the role of one agent and information of other agents, represented in a causal graph. Theoretically, we show this constrained objective can be decomposed into two terms: 1) maximizing the mutual information between the role of one agent and the information about other agents to capture the causal effect in role assignment, enabling reasonable role assignment over the corresponding causal graph, and 2) maximizing the heterogeneity of roles in the determinant form to yield more coherent role clusters without redundancy. These two terms can further be converted into intrinsic rewards. Interpreted as the intrinsic reward, CORD can be implemented by extending QMIX \citep{qmix} or REFIL \citep{refil} and trained end-to-end by optimizing the shaped rewards to enable generalizable cooperation across different teams with unseen agents.

Empirically, we evaluate CORD in a variety of environments including resource collection \citep{copa} in multi-agent particle environments (MPE) \citep{mpe} and multi-task StarCraft Multi-Agent Challenge (SMAC) \citep{smac,refil}. Results show that CORD outperforms baselines and achieves better performance in generalization tests. By ablation studies, we verify the effectiveness of our proposed optimization objective.


\section{Related Work}
\label{section: related work}
\subsection{Hierarchical Reinforcement Learning}
Hierarchical RL \citep{al2015hierarchical} solves a complex task by hierarchically decomposing it into simpler sub-tasks. In single-agent settings, the high-level controller selects options \citep{bacon2017option, sutton1999between, precup2000temporal}, reusable skills \citep{daniel2012hierarchical, gregor2016variational, sharma2019dynamics, shankar2020learning} or subgoals \citep{levy2019learning, sukhbaatar2018learning, nachum2018near, dwiel2019hierarchical, nasiriany2019planning} for the low-level policy to solve long-horizon tasks. Recent MARL studies have employed hierarchical frameworks to address team composition problems. For example, COPA \citep{copa} proposes a coach-player framework, where the controller learns a strategy distribution for the low-level agents based on global information, without considering team dynamics. ALMA \citep{iqbalalma} utilizes human domain knowledge to pre-define many subtasks and corresponding rewards for learning a high-level subtask allocation policy that assigns subtasks to agents. HSL \citep{liu2022heterogeneous} employs an auto-encoder model to develop representations under a fixed number of diverse skills, enabling agents to select different skills as needed. Unlike existing work, CORD exploits an informative posterior role distribution to learn the role assignment for the worker agents.

\subsection{Role-Based Reinforcement Learning}

Roles are associated with the division of labor and the key to multi-agent systems \citep{8780531, campbell2011multi}. Based on this intuition, many methods have been proposed to leverage predefined role assignments to solve specified tasks \citep{lhaksmana2018role, sun2020reinforcement}. However, predefined roles need prior knowledge which hurts generalization. To solve this problem, \citet{wilson2010bayesian} learns roles by Bayesian inference. ROMA \citep{roma} distinguishes role distributions based on observation trajectories by mutual information. ROMA does not character more complex inter-agent interactions in environments. RODE \citep{rode} maintains the original framework of HSL while replacing skills with roles. These methods only utilize the learned role distributions to promote cooperation in a fixed team. Unlike these methods, CORD optimizes the role assignment based on the maximum entropy principle to promote generalization across different teams with unseen agents.

\subsection{Multi-Agent Generalization}

Zero-shot coordination has been a widely studied problem in multi-agent systems. It refers to the ability of effectively cooperating with unseen agents. However, previous studies \citep{hu2020other, yu2023improving, lupu2021trajectory} often assume a fixed number of agents and cannot handle scenarios with variable numbers of agents. To achieve robust behaviors among varying numbers of unknown teammates in multi-agent cooperation, generalization problems \citep{stone2010ad, zhang2020multi, mahajan2022generalization} have received much attention. Existing type-inference approaches assume a finite set of predefined teammate types and choose policies adaptively to solve generalization. For example, PLASTIC \citep{barrett2015cooperating} computes the Bayesian posterior of predefined types. AATEAM \citep{chen2020aateam} proposes an attention-based architecture to infer the type. As these methods assume predefined teammate types, they cannot generally generalize to unknown types. Some recent works avoid predefining the type of teammates via complex training processes, such as population-based training \citep{long2020evolutionary}, pre-training \citep{xing2021learning, gu2021online}, and adversarial training \citep{li2019robust}. Other studies leverage communication. SOG \citep{shao2022self} designs different mechanisms for agent self-organization into teams based on the prior that agents with similar observations should communicate. SOG links agent observation trajectories with communication by optimizing a variational lower bound on mutual information. However, these methods require manual design, pre-training, or communication, which brings additional complexity to real-world scenarios. Unlike these methods, CORD addresses the generalization problem without reliance on predefined agent policies/behaviors or communication, and can be trained in an end-to-end manner.



\section{Background}
\label{section: background}
\textbf{Dec-POMDP.} A decentralized partially observable Markov decision process (Dec-POMDP) \citep{oliehoek2016concise} can be defined as a tuple: $ \langle \mathbf{S}, \mathbf{A}, \mathcal{P}, r, \mathbf{U}, \Phi, O, N, \gamma \rangle$, where $N$ is the number of agents, $\mathbf{U}=\{1,2...,N\}$ is the set of agents. $\mathbf{S}$ is the set of the states, and $\mathbf{A}$ is the set of joint actions, $\boldsymbol{a} = \{\boldsymbol{a}^i | i \in \mathbf{U}\}$ $\in \mathbf{A}$. At each state $\mathbf{s} \in \mathbf{S}$, each agent $i$ receives a partial observation $\mathbf{o}^i \in \Phi$ according to the observation function $O(\mathbf{s}, i):\mathbf{S}\times \mathbf{U} \to \Phi$. Agents choose their actions forming a joint action $\boldsymbol{a} \in \mathbf{A}$. The joint action causes a transition to the next state $\mathbf{s}'$ according to the state transition function $\mathcal{P}(\mathbf{s}'|\mathbf{s},\boldsymbol{a})): \mathbf{S} \times \mathbf{A} \times \mathbf{S} \to [0,1]$, and the global reward of the team is determined by $r(\mathbf{s}, \boldsymbol{a})): \mathbf{S} \times \mathbf{A} \to \mathbb{R}$. To settle partial observability, the trajectory of each agent $\boldsymbol{\tau}^i \in \mathbf{T}: (\Phi \times \mathbf{A})^{\ast}$ is used to replace observation $\mathbf{o}^i$. Each agent learns a local policy $\boldsymbol{\pi}^i(\boldsymbol{a}^i|\boldsymbol{\tau}^i)$, and all together form a joint policy $\boldsymbol{\pi}(\boldsymbol{a}|\boldsymbol{\tau})$, where $\boldsymbol{\tau}$ is the joint trajectory of all agents. The objective is to learn a joint policy to maximize the expected cumulative discounted return, $\mathbb{E}[\sum_{t=0}^\infty \gamma^t r_t]$, where $\gamma \in [0, 1)$ is the discount factor. For a joint policy $\boldsymbol{\pi}(\boldsymbol{a}|\boldsymbol{\tau})$, we can define the joint state-action value function $Q^{tot}(\boldsymbol{\tau}_t, \boldsymbol{a}_t) = \mathbb{E}_{\boldsymbol{\tau}_{t+1:\infty}, \boldsymbol{a}_{t+1:\infty}}[\sum_{k=0}^{\infty}\gamma^{t}r_{t+k}|\boldsymbol{\tau}_t, \boldsymbol{a}_t]$. 
Further we denote all other agents except agent $i$ as $-i$.

\textbf{QMIX.} Value decomposition methods \citep{sunehag2018value,qmix,son2019qtran,wang2020qplex} factorize the joint state-action value function $Q^{tot}$ into individual state-action functions to solve Dec-POMDPs. QMIX \citep{qmix} is one of the commonly used methods, and it factorizes $Q^{tot}$ into $\{Q^i(\boldsymbol{\tau}^i, \boldsymbol{a}^i) | i \in \mathbf{U}\}$ via a mixing network which satisfies the IGM condition by $\frac{\partial Q^{tot}}{\partial Q^i} \ge 0, \ \forall I \in \mathbf{U}$. The mixing network can be removed during execution such that agents can make decisions by their own $Q^i$ in a fully decentralized manner. To handle a variable number of agents, Attention QMIX (AQMIX) \citep{refil} improves QMIX by the multi-head attention mechanism~\citep{vaswani2017attention}. REFIL \citep{refil} is further introduced to enhance AQMIX for multi-task MARL by random subgroup partition.  


\section{CORD}
\label{section: method}

Inspired by information theory \citep{jaynes1957information}, we believe that increasing the entropy of the role distribution is an effective approach to resolve the problem of generalizable cooperation with unseen agents. The objective function can be considered as,
\begin{align}\label{eq3}
    \max \mathcal{H}(\mathcal{P(\mathbf{c})}),
\end{align}
where $\mathcal{P(\mathbf{c})}$ is a prior role distribution. Without any task information, the best prior distribution would be a max entropy distribution. When the task information is given, we can further adapt this prior based on the task information and get a more informative posterior. Therefore, we conjecture that merely maximizing the entropy of the role distribution is insufficient for learning high-quality role division. Our empirical results (see Section \ref{sec_ablation}) also support that.

To derive an informative posterior role distribution by maximizing the entropy, some evidence or information should be provided to the posterior. \textit{We believe that a good role assignment for cooperation should consider the information or influence from other agents.} In a role-based approach, effective role assignment is crucial for fostering cooperation. As highlighted by \citet{chalkiadakis2022computational, elkind2016cooperative, branzei2008models}, considering the impact of individual agents, such as their Shapley value, is key to enhancing team cooperation. An illustrative example: A robot team of excavators, material transporters, and assemblers collaborate to build, where delayed excavation due to hard rocks prompts some transporters to switch roles to clear debris, accelerating the process. Thus, effective role assignment is responsive to the excavators' performance.


Thus, we assume the presence of a causal relationship between the role of one agent and the information of other agents. Moreover, we construct a causal graph to reflect such relationships. With these conditions and the corresponding causal graph, we can find that the entropy of role distribution \eqref{eq3} can be split into two terms: the first term is causal inference in role and the second term is role heterogeneity. We discuss this result in detail later in this section. 

Before our discussion of the decomposition of the entropy of role distribution, we must first establish a reasonable causal graph. We assume that a causal relationship exists between the role of one agent and the influence of other agents. As such, to ascertain the precise impact of these relationships on the role distribution, we must quantify them in terms of mutual information based on the corresponding causal graph. For this purpose, we put forth the assumptions below to facilitate the construction of the causal graph.

\begin{assumption}[\textbf{Causal Graph Assumption}]\label{a2}
    At timestep $t$, the actions $\boldsymbol{a}_{t}^{-i}$ and other agents' roles $\mathbf{c}_{t}^{-i}$ cannot influence agent $i$'s role $\mathbf{c}_{t}^i$. Thus, we consider $\boldsymbol{a}_{t-1}^{-i}$ and $\mathbf{c}_{t-1}^{-i}$, which we posit having a causal relationship with $\mathbf{c}_{t}^i$. Given that we assume environments satisfy Markov Properties, we need only consider the observation at timestep $t$, denoted $\mathbf{o}_t^i$, which encapsulates all prior information.
\end{assumption}

\begin{figure}[t!]
\centering
\includegraphics[width=0.98\linewidth]{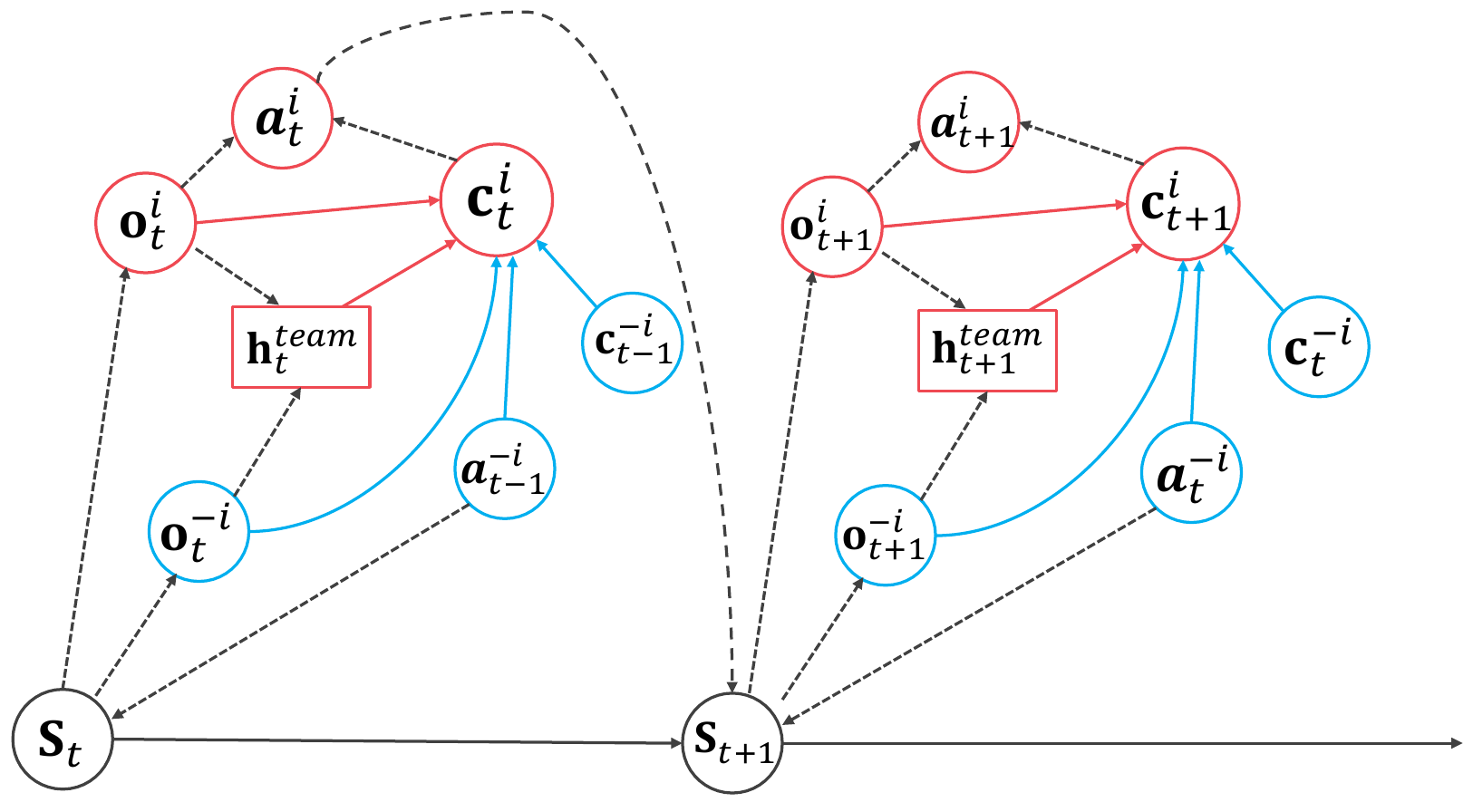}
\caption{Illustration of the causal graph. Black circles are inherent states in the environment. Blue circles represent related information of other agents except for agent $i$. Red circles are related information of agent $i$. Black dashed lines represent the direction of information transmission. Blue solid lines represent the causal effect on the role of agent $i$ and red solid lines are the individual information.}
\label{fig2:causalgraph}
\end{figure}

Under Assumption \ref{a2}, we propose a causal graph as illustrated in Figure~\ref{fig2:causalgraph}. In this causal graph, for any agent $i$, the observation of agent $i$ at timestep $t$: $\mathbf{o}_t^i$ and team state $\mathbf{h}^{team}_t=\{\mathbf{o}^k_t|k \in \mathbf{U}\}$ can represent the sufficient information of agent $i$. The observation of other agents at timestep $t$: $\mathbf{o}_t^{-i}$, the action and role distribution of other agents at timestep $t-1$: $\boldsymbol{a}_{t-1}^{-i}$ and $\mathbf{c}_{t-1}^{-i}$ can represent the causal influence on agent $i$.

Now, given this causal graph, we can define some critical quantities for our analysis. First, we can get the following definition about the other agents' influence.

\begin{definition}\label{d1}
    Suppose that the causal relationships between agents can be defined as the causal graph, for each agent $i$, the definition of other agents' influence vector $\boldsymbol{\bar{I}}_t^i$ is below:
    \begin{align}
    \begin{split}
        \mathbf{q}_t^i &= f(\mathbf{o}_t^i, \mathbf{h}_t^{team}),\\ 
        \mathbf{k}_t^{i} &= g(\boldsymbol{a}_{t-1}^{i}, \mathbf{o}_t^{i}, \mathbf{c}_{t-1}^{i}),\\ 
        \mathbf{v}_t^{i}&=h(\boldsymbol{a}_{t-1}^{i}, \mathbf{o}_t^{i}, \mathbf{c}_{t-1}^{i}), \\ 
        \alpha^{j}_t &= \mathrm{softmax}( \frac{\mathbf{q}_t^i \cdot \mathbf{k}_t^{j}}{\sqrt{d}}), \forall j \neq i \\ 
        \boldsymbol{\bar{I}}_t^i &=\sum_{j \neq i} \alpha^j_t \mathbf{v}_t^j,
    \end{split}
    \end{align}
where $\mathbf{q}_t^i$ is encoded by $f$ function, $\mathbf{k}_t^{i}$ and $\mathbf{v}_t^{i}$ are the key and value vector encoded by $g$ and $h$ function respectively, $f$, $g$, and $h$ are trainable network layers with different parameters. 
\end{definition}

Here we follow the attention mechanism \citep{bahdanau2014neural} to process the information about the influence of other agents $-i$ on agent $i$. Thus, Definition \ref{d1} represents the specific computation process of defining the influence vector of an agent by the attention mechanism.

Next, we can define $\mathbf{c}$-related matrix $A(\mathbf{c})$ which represents the similarities between different roles.

\begin{definition}\label{d2}
    There are a multivariate role distribution $\mathcal{P(\mathbf{c})}$ and $\mathbf{c}$-related Matrix $A(\mathbf{c})$. Their definitions are as follows:
    \begin{align}
        \notag A(\mathbf{c})_{ij} &= e^{-d_{ij}}, \ A(\mathbf{c})_{ij} \in (0, 1], \ \forall i, j \in \{1,2,..., N\}, \\ 
        \notag d_{ij} &= D_{\mathrm{KL}}(\mathcal{P}(\mathbf{c}_i|\boldsymbol{\bar{I}}_t^i, \mathbf{q}_t^i)||\mathcal{P}(\mathbf{c}_j|\boldsymbol{\bar{I}}_t^j, \mathbf{q}_t^j))\\ 
        &\quad +D_{\mathrm{KL}}(\mathcal{P}(\mathbf{c}_j|\boldsymbol{\bar{I}}_t^j,
        \mathbf{q}_t^j))||\mathcal{P}(\mathbf{c}_i|\boldsymbol{\bar{I}}_t^i, \mathbf{q}_t^i),
    \end{align}
where $\boldsymbol{\bar{I}}_t^i$, $\mathbf{q}_t^i$, $\boldsymbol{\bar{I}}_t^j$, and $\mathbf{q}_t^j$ satisfy Definition~\ref{d1}, and $N$ is the number of agents. 
\end{definition}

With all the preparation above, we can discuss the decomposition of the role entropy. Actually, we have the following theorem.

\begin{theorem}\label{t1} 
    Suppose that both the prior role distribution $\mathcal{P(\mathbf{c}|\mathbf{q})}$ and the posterior role distribution $\mathcal{P}(\mathbf{c}|\boldsymbol{\bar{I}}, \mathbf{q})$ obey Gaussian distribution and the $\mathbf{c}$-related matrix $A(\mathbf{c})$ satisfies Definition~\ref{d2}, then the entropy of the role distribution can be decomposed as:
    \begin{align}
        & \mathcal{H}(\mathcal{P(\mathbf{c}|\mathbf{q})}) = \mathcal{I}(\mathbf{c};\boldsymbol{\bar{I}}|\mathbf{q}) + \mathcal{H}(\mathcal{P}(\mathbf{c}|\boldsymbol{\bar{I}}, \mathbf{q})), \label{f7} \\
        & \mathcal{I}(\mathbf{c};\boldsymbol{\bar{I}}|\mathbf{q}) = \mathbb{E}_{\boldsymbol{\tau}}[\sum^{N}_{i=1}D_{\mathrm{KL}}[\mathcal{P}(\mathbf{c}_i|\boldsymbol{\bar{I}}_i, \mathbf{q}_i)\|\mathcal{P}(\mathbf{c}_i|\boldsymbol{do(\bar{I}}_i), \mathbf{q}_i)]], \label{eq:causal} \\
        & \mathcal{H}(\mathcal{P}(\mathbf{c}|\boldsymbol{\bar{I}}, \mathbf{q})) = \beta \log |A(\mathbf{c})| + C, \label{eq:role-H}
    \end{align}
where $\mathcal{I}$ is mutual information, $\boldsymbol{\bar{I}}_i = \boldsymbol{\bar{I}}_t^i$, $\mathbf{q}_i = \mathbf{q}_t^i$, $\boldsymbol{do(\bar{I}}_i) = \boldsymbol{do(\bar{I}}_t^i)$, $|A(\mathbf{c})|$ denotes the determinant of $A(\mathbf{c})$, and $\beta$ and $C$ are constants.
\begin{proof}
See Appendix~\ref{SPI_proof}.
\end{proof}
\end{theorem}

$\boldsymbol{do(\bar{I}}_i)$ is a mathematical operator which represents the average influence from other agents, so $D_{\mathrm{KL}}[\mathcal{P}(\mathbf{c}_i|\boldsymbol{\bar{I}}_i, \mathbf{q}_i)\|\mathcal{P}(\mathbf{c}_i|\boldsymbol{do(\bar{I}}_i), \mathbf{q}_i)]$ can quantify counterfactual causal effects from other agents. The difference between ``do-calculus'' and $\boldsymbol{do(\bar{I}}_i)$ is that ``do-calculus'' is the operation of intervention, while $\boldsymbol{do(\bar{I}}_i)$ refers to the expectation of all possible interventions. Practically, $\boldsymbol{do(\bar{I}}_i)$ can be substituted with a constant vector. More discussion about $\boldsymbol{do(\bar{I}}_i)$ can be found in \citet{pearl2009causal}. However, instead of using the method in \citet{pearl2009causal} to calculate the causal effect, we borrow the idea of social influence \citep{jaques2019social} and leverage the deep learning model to estimate the causal effect.

\begin{figure*}[t]
\centering
\includegraphics[width=0.68\linewidth]{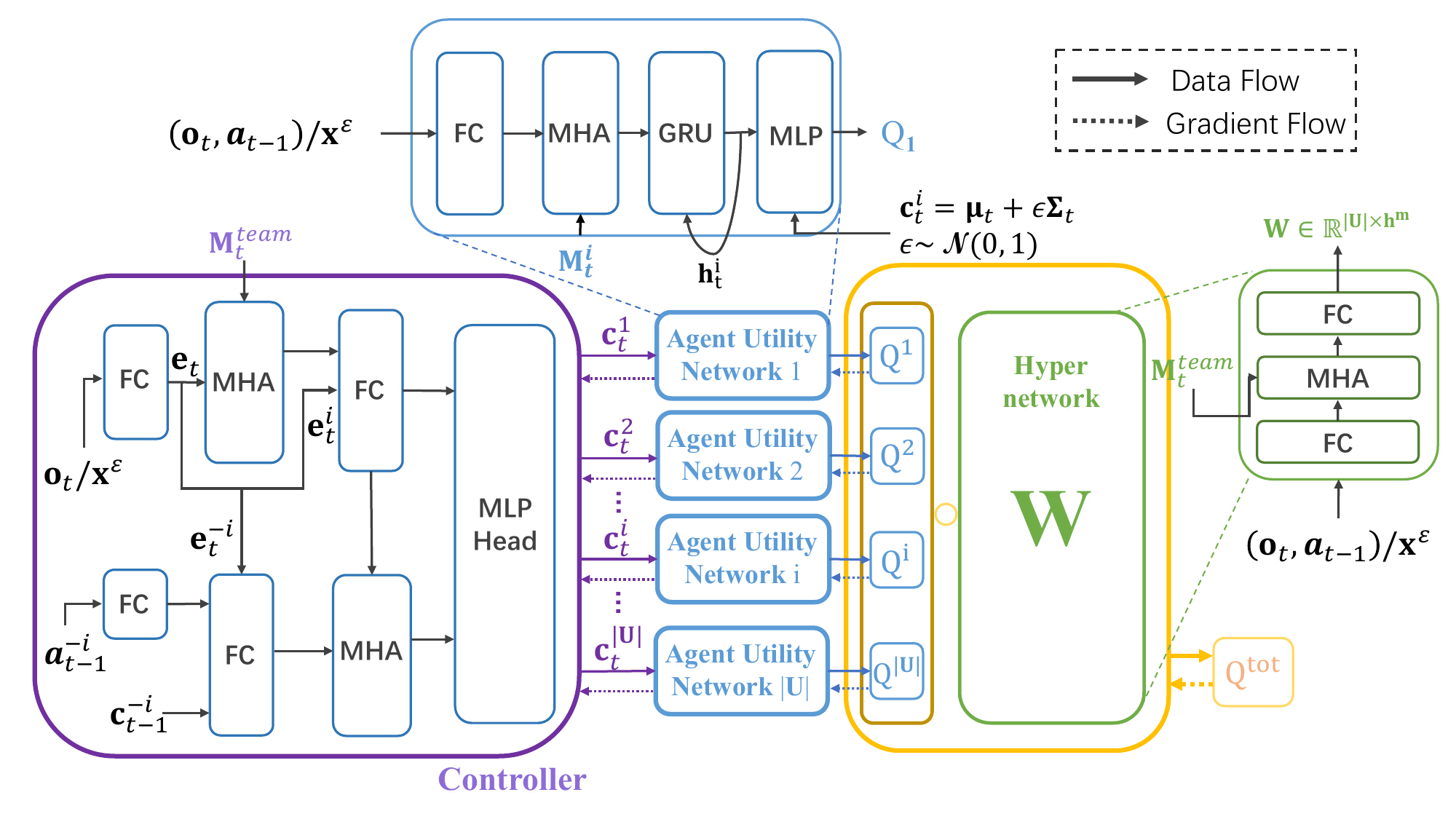}
\caption{Overview of CORD framework. The purple module is the high-level controller network. The blue modules represent agent individual Q-networks and the orange module is the mixing network. To additionally settle entities-based settings, for the network input, $(\mathbf{o}_t, \boldsymbol{a}_{t-1})/\mathbf{x}^\mathcal{E}$ represents different types of observation for different entities, and $\mathbf{M}^{team}_t$ and $\mathbf{M}_t^i$ are masks of team state and observation of agent $i$ respectively.}
\label{fig3:framwork}
\end{figure*}

From Theorem \ref{t1}, we can find that the entropy of role distribution is split into two terms: \textit{causal inference in role} $\mathcal{I}(\mathbf{c};\boldsymbol{\bar{I}}|\mathbf{q})$ and \textit{role heterogeneity} $\mathcal{H}(\mathcal{P}(\mathbf{c}|\boldsymbol{\bar{I}}, \mathbf{q}))$. The causal inference of roles $\mathcal{I}(\mathbf{c};\boldsymbol{\bar{I}}|\mathbf{q})$  enables prudent role assignments through causal reasoning based on the previously defined causal graph. The role heterogeneity $\mathcal{H}(\mathcal{P}(\mathbf{c}|\boldsymbol{\bar{I}}, \mathbf{q}))$ incentivizes the controller to derive dissimilar role partitions without undesirable redundancy. Thus, the model can maintain its performance while enhancing generalization by concurrently optimizing both components.

We need to argue that the objective \eqref{f7} is different from \eqref{eq3}. The decomposition in Theorem \ref{t1} is under Definition \ref{d1} and the causal graph. These assumptions or conditions can be seen as constraints on optimizing the entropy of the role distribution in \eqref{eq3}. So optimizing the objective \eqref{f7} is actually optimizing \eqref{eq3} with some constraints from the causal graph.

Though $\mathcal{I}(\mathbf{c};\boldsymbol{\bar{I}}|\mathbf{q})$ and $\mathcal{H}(\mathcal{P}(\mathbf{c}|\boldsymbol{\bar{I}}, \mathbf{q}))$ are components of the decomposition of $\mathcal{H}(\mathcal{P(\mathbf{c})})$ from Theorem \ref{t1}, these two terms still have richer meanings from different views. Actually, these different views are the reason that we name them as causal inference in role and role heterogeneity.  So, we discuss these two terms in further depth next.

\subsection{Causal Inference in Role}

To enable prudent role assignments, the precise causal impact based on the previously defined causal graph must be ascertained. Actually, the RHS of \eqref{eq:causal} represents the expected causal effects from other agents \citep{pearl2009causal}. The proof demonstrating how causal inference relates to mutual information is provided in Appendix~\ref{SPI_proof}. So for the purpose of prudent role assignments, assuming the role distribution satisfies Definition~\ref{d1}, the objective function can be defined as,
\begin{align}\label{f8}
        \max_{\mathbf{c} \sim \mathcal{P}(\cdot|\boldsymbol{\bar{I}}, \mathbf{q}) } \mathbb{E}_{\boldsymbol{\tau}}[\sum^{N}_{i=1}D_{\mathrm{KL}}[\mathcal{P}(\mathbf{c}_i|\boldsymbol{\bar{I}}_i, \mathbf{q}_i)||\mathcal{P}(\mathbf{c}_i|\boldsymbol{do(\bar{I}}_i), \mathbf{q}_i)]].
\end{align}
Moreover, from Theorem \ref{t1}, maximizing \eqref{f8} is also equivalent to maximizing the mutual information between the role and the influence vector of other agents. Note that mutual information can be used in causal inference, including counterfactual inference and intervention operations. \citet{NEURIPS2022_98a5c047} explains that mutual information discovers causal relationships by observation data and evaluates the impact of a variable change on another variable.

Although \eqref{f8} optimizes the expected value of causal effects, practically we can sample $M$ trajectories and approximate the mutual information between $\mathbf{c}_i$ and $\boldsymbol{\bar{I}}_t^i$ by the average causal effects. Moreover, we substitute $\boldsymbol{I}_0$, a constant vector, for $\boldsymbol{do}(\boldsymbol{\bar{I}}^i_t)$ to represent the intervention in causal inference theory \citep{pearl2009causal}. Therefore, the causal inference in role can be rewritten as follows, and we maximize it by taking it as an intrinsic reward,
\begin{align}\label{r_c}
    r_c = \frac{1}{M}\sum_{M}\sum^{N}_{i=1}D_{\mathrm{KL}}[\mathcal{P}(\mathbf{c}_i|\boldsymbol{\bar{I}}_i, \mathbf{q}_i)||\mathcal{P}(\mathbf{c}_i|\boldsymbol{I}_0, \mathbf{q}_i)],
\end{align}
where $\boldsymbol{I}_0$ is a constant vector, $M$ is the number of sampled trajectories.

\begin{figure*}[t!]
\centering
\begin{subfigure}{0.3\textwidth}
    \centering
    \includegraphics[width=.956\linewidth]{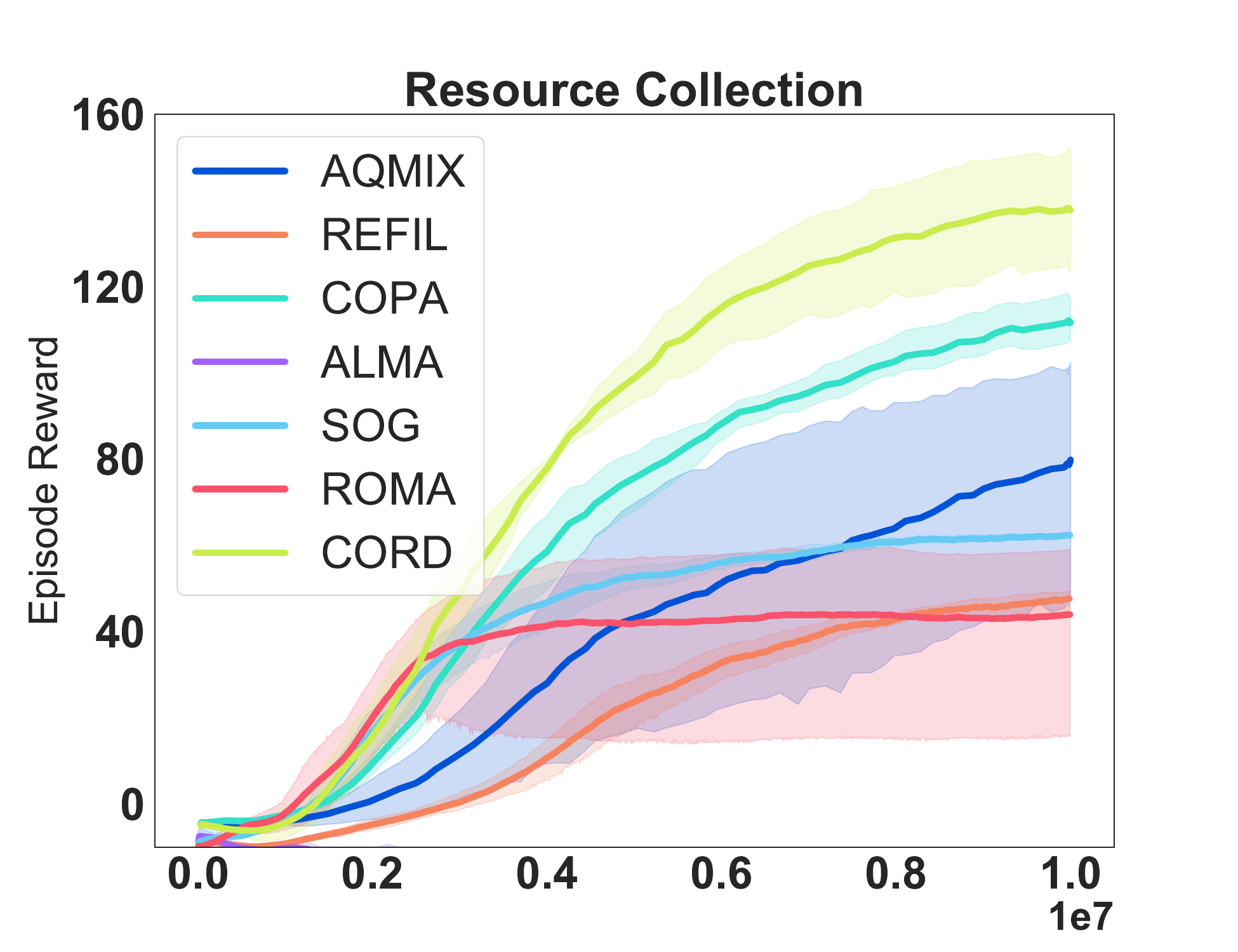}
    \caption{learning curves}
    \label{Fig4.sub.1} 
\end{subfigure}
\quad
\begin{subfigure}{0.3\textwidth}
    \centering
    \includegraphics[width=.96\linewidth]{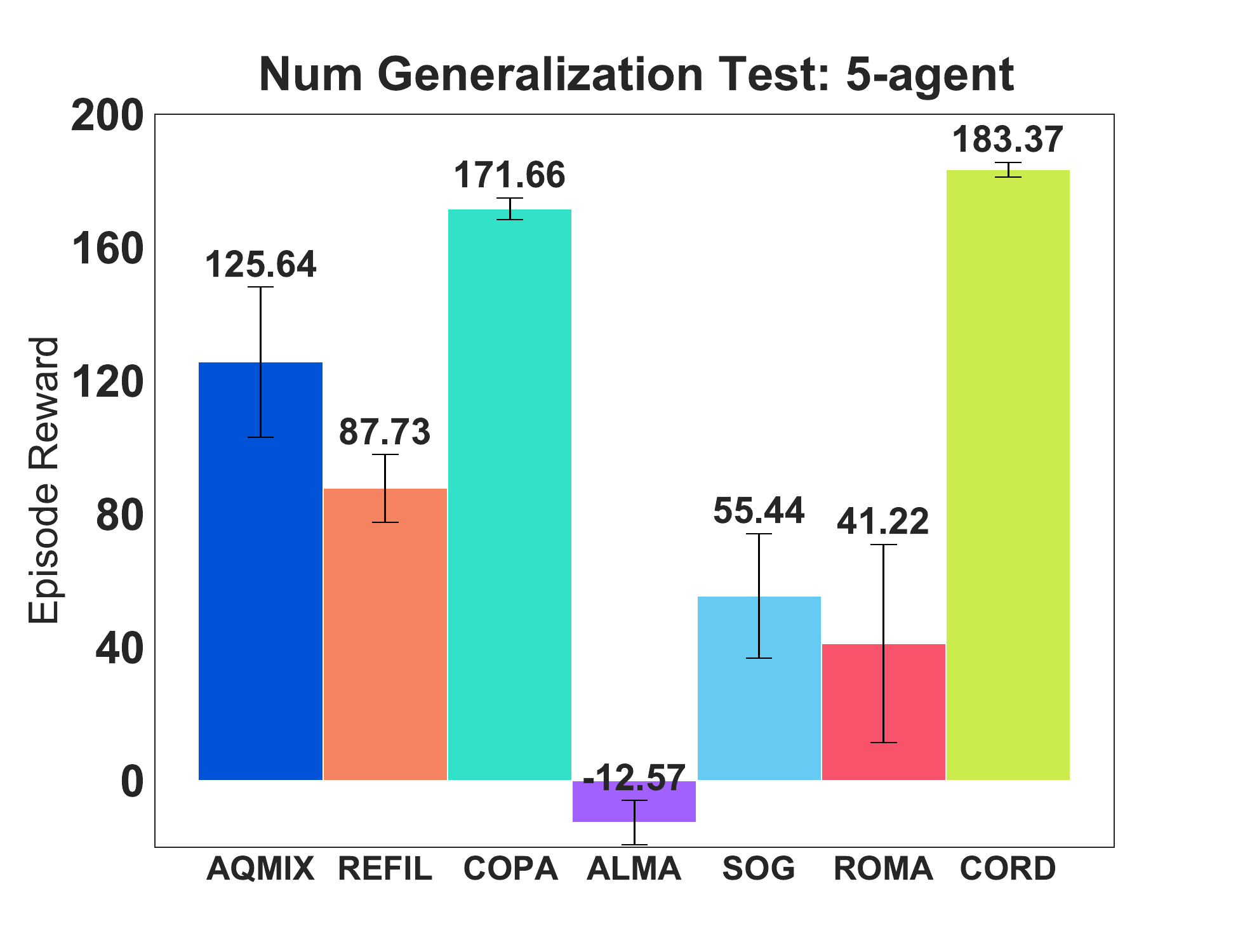}
    \caption{5-agent task}
    \label{Fig4.sub.2} 
\end{subfigure}
\quad
\begin{subfigure}{0.3\textwidth}
    \centering
    \includegraphics[width=.99\linewidth]{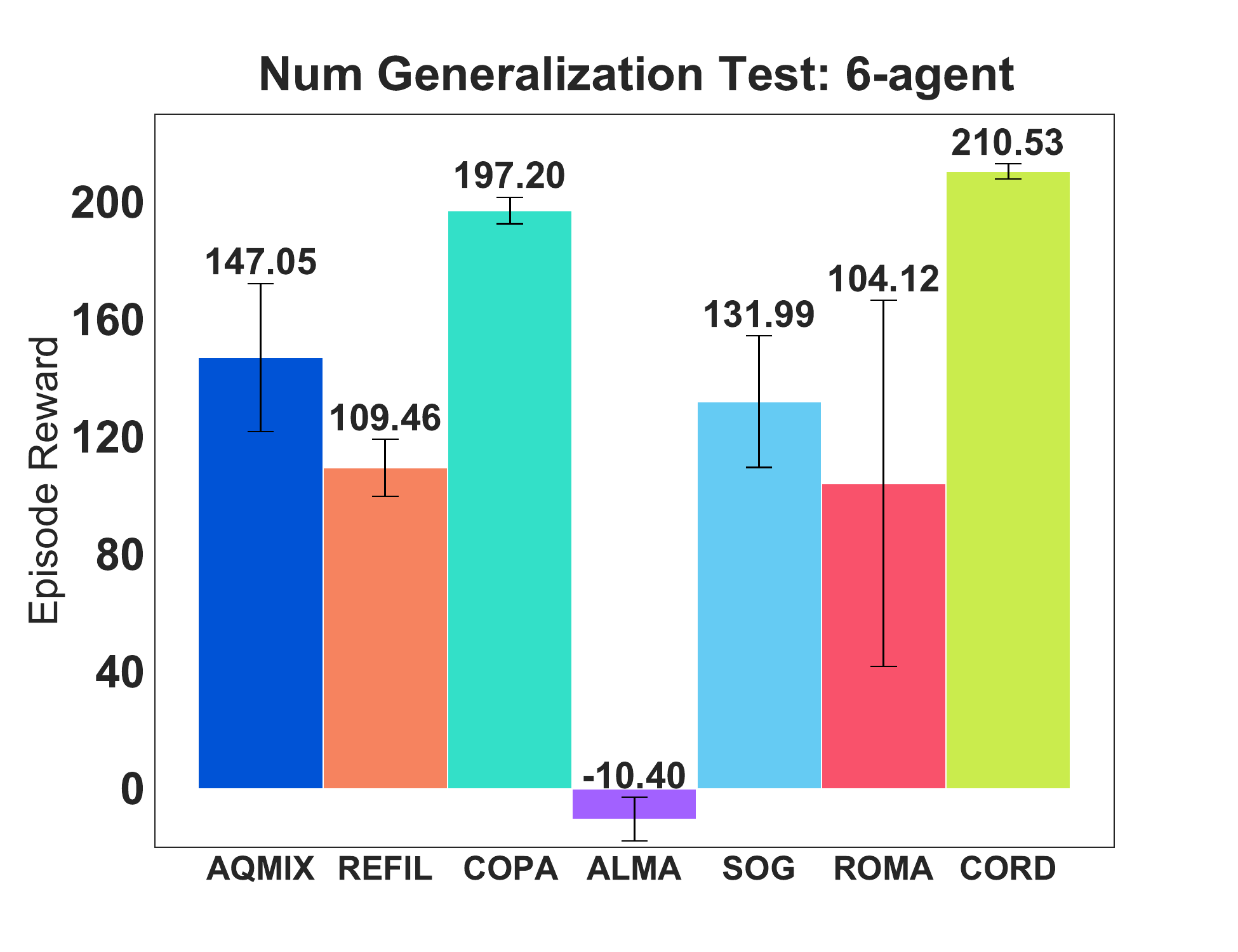}
    \caption{6-agent task}
    \label{Fig4.sub.3}
\end{subfigure}
\caption{Episode rewards on training and generalization to \textit{\textbf{unseen teams}} of CORD compared with all baselines in resource collection: \subref{Fig4.sub.1} learning curves on training tasks, \subref{Fig4.sub.2} generalization to 5-agent team, and \subref{Fig4.sub.3} generalization to 6-agent team.}
\label{Fig4.resource}
\end{figure*}

\subsection{Role Heterogeneity}
Given Definition \ref{d2} and the assumption of Gaussian distribution, we can find the determinant $|A(\mathbf{c})|$ in the formulation of $\mathcal{H}(\mathcal{P}(\mathbf{c}|\boldsymbol{\bar{I}}, \mathbf{q}))$ in \eqref{eq:role-H}. On the other hand, the matrix $A(\mathbf{c})$ depicts the distance between any pair of roles. This means that the determinant $|A(\mathbf{c})|$ represents the enclosed volume of $\mathbf{c}$ in the corresponding metric space, and maximizing $|A(\mathbf{c})|$ improves the diversity of roles \citep{parker2020effective}. So the following objective function is equivalent to maximizing the entropy of posterior role distributions given other agents' influence vectors, as well as improving the diversity of roles,
\begin{align}\label{f10}
    \max_{\mathbf{c} \sim \mathcal{P}(\cdot|\boldsymbol{\bar{I}}, \mathbf{q})} \beta\log |A(\mathbf{c})| + C. 
\end{align}
However, it should be noted that $\log |A(\mathbf{c})| \in (-\infty, 0]$ is unbounded at one end, while $|A(\mathbf{c})| \in [0, 1]$ is bounded. Since the $\log$ function is concave, we optimize $\log |A(\mathbf{c})|$ in the same manner as we optimize $|A(\mathbf{c})|$. Thus, the role heterogeneity is expressed as follows and maximized as an intrinsic reward,
\begin{align}\label{r_d}
    r_d = |A(\mathbf{c})|.
\end{align}

\subsection{Framework}
\label{frameworkSection}

We are now ready to introduce our learning framework. First, by aggregating the two intrinsic rewards $r_c$ and $r_d$ with the environmental reward $r_e$, we can have a new reward function $r$ as follows:
\begin{align}\label{newReward}
    r = r_e + \lambda_c r_c + \lambda_d r_d,
\end{align}
where $\lambda_c$ and $\lambda_d$ are hyperparameters. Our objective is to optimize this shaped reward. The learning framework of CORD is illustrated in Figure \ref{fig3:framwork}. The high-level controller takes as input the observations from all agents and then assigns roles to the low-level agents accordingly. 
The agent utility network computes individual Q-function given local observation and assigned role. The mixing network takes as input the Q-values from all agents and outputs $Q^{tot}$. All the modules, parameterized by $\theta$, are updated end-to-end via backpropagation to minimize the TD loss, 
\begin{align}
\begin{split}
    \mathcal{L}(\theta) &= \mathbb{E}_{(\boldsymbol{\tau},\boldsymbol{a}),r,\boldsymbol{\tau}')\sim \mathcal{D}} \left[\left(y^{tot}-Q^{tot}(\boldsymbol{\tau}, \boldsymbol{a}); \theta)\right)^2 \right],\\
    y^{tot} &= r + \gamma Q^{tot}\left(\boldsymbol{\tau}', \arg\max Q^{tot}(\boldsymbol{\tau}',\cdot \ ; \theta); {\bar \theta} \right),
\end{split}
\label{eq2}
\end{align}
where $\mathcal{D}$ is the replay buffer and $\bar \theta$ is the parameter of the target network. The pseudocode of the learning algorithm is available in Appendix~\ref{app:algorithm}.

\begin{figure*}[t!]
\centering
\begin{subfigure}{0.228\textwidth}
	\includegraphics[width=\linewidth]{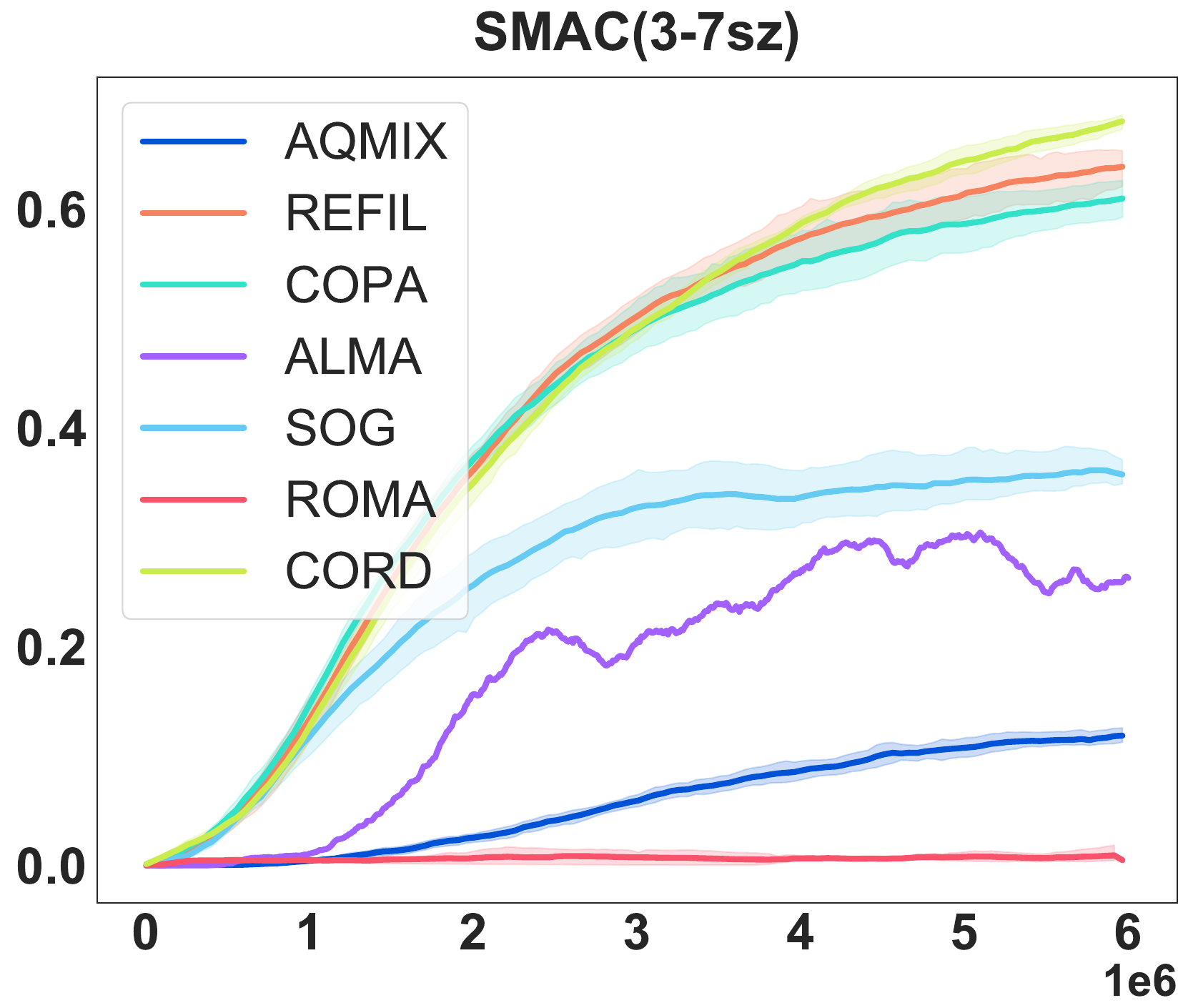}
    \caption{3-7sz}\label{Fig5.sub.1}
\end{subfigure}
\quad
\begin{subfigure}{0.228\textwidth}
	\includegraphics[width=\linewidth]{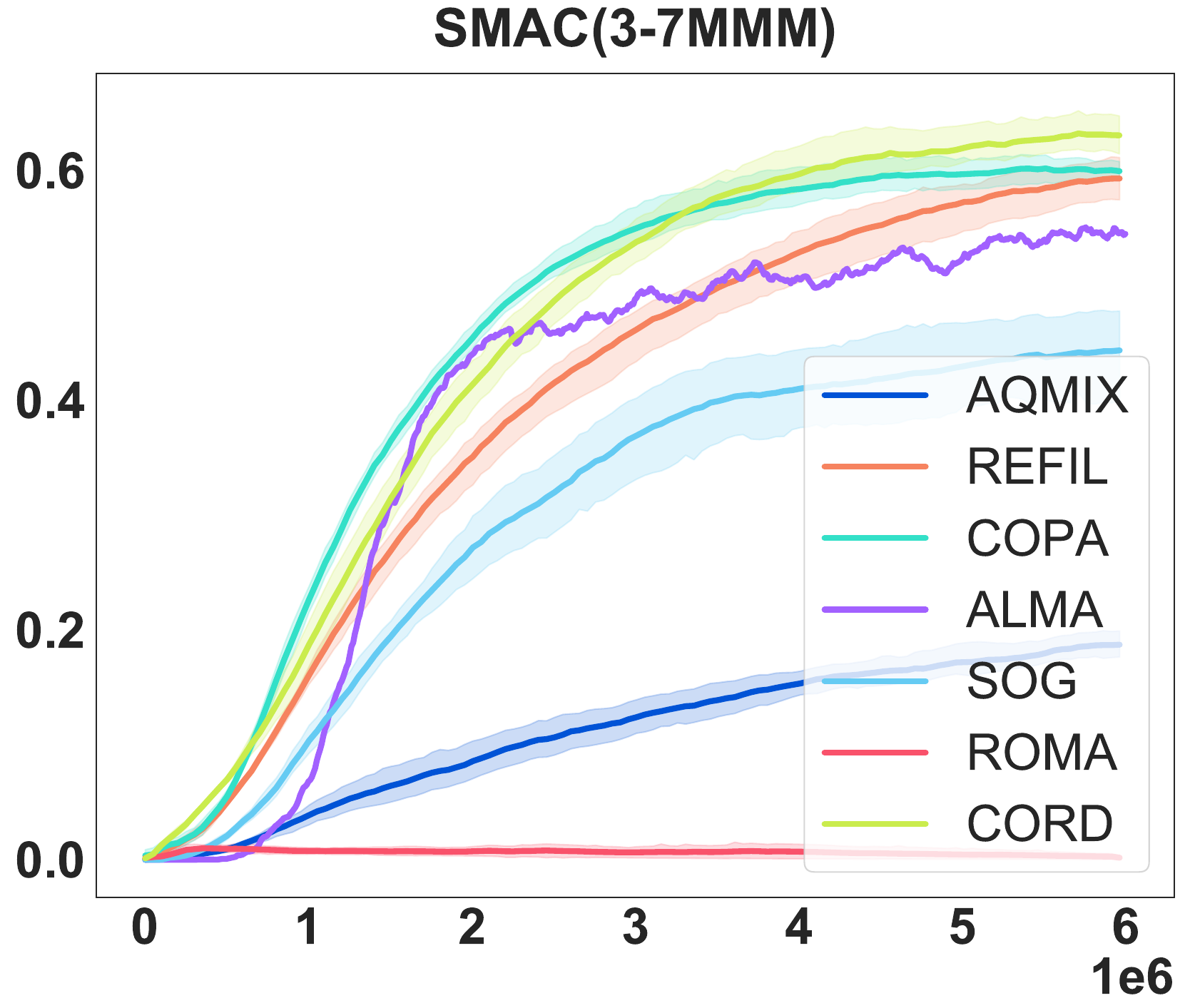}
    
    \caption{3-7MMM}
    \label{Fig5.sub.2}
\end{subfigure}
\quad
\begin{subfigure}{0.228\textwidth}
    \centering
	\includegraphics[width=\linewidth]{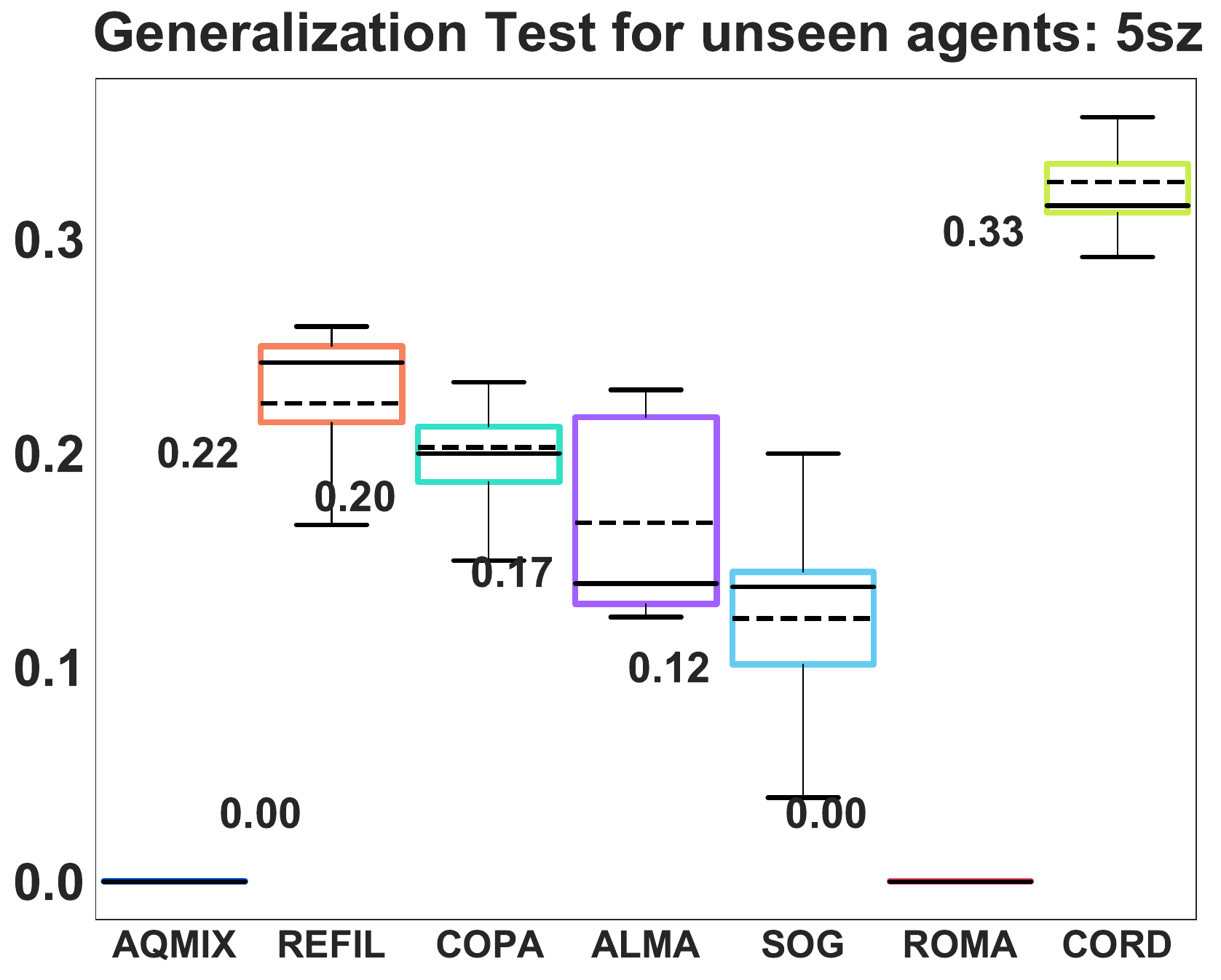}
    \captionsetup{skip=10pt}
    \caption{5sz}\label{Fig5.sub.3}
\end{subfigure}
\quad
\begin{subfigure}{0.238\textwidth}
    \centering
	\includegraphics[width=\linewidth]{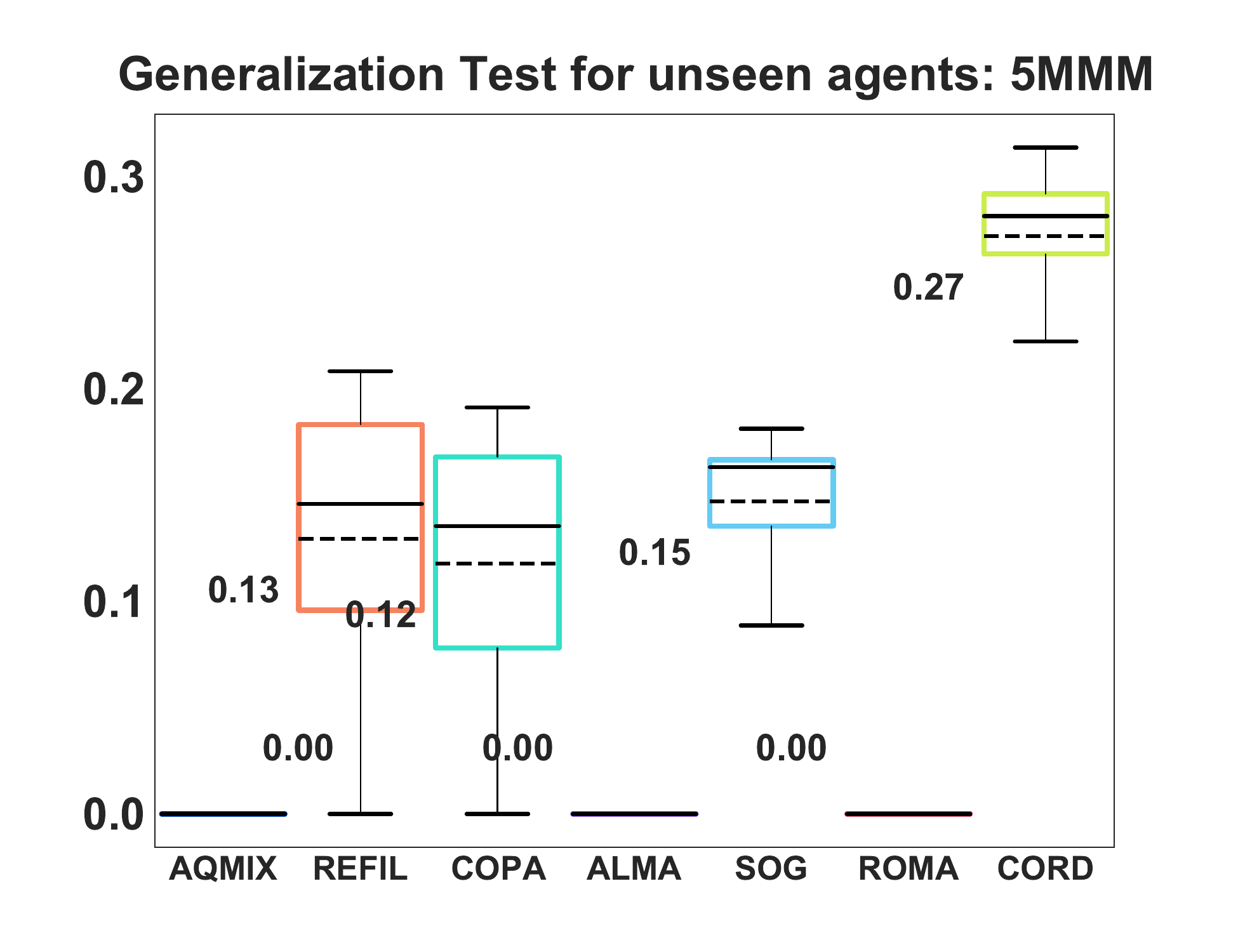}
    \captionsetup{skip=10pt}
    \caption{5MMM}\label{Fig5.sub.4}
\end{subfigure}
\caption{
Win rates on training and generalization to \textit{\textbf{unseen agents}} of CORD compared with all baselines in SMAC: \subref{Fig5.sub.1} learning curves in 3-7sz, \subref{Fig5.sub.2} learning curves in 3-7MMM, \subref{Fig5.sub.3} generalization to 5sz, and \subref{Fig5.sub.4} generalization to 5MMM, where \subref{Fig5.sub.3} and \subref{Fig5.sub.4} are the box plot.}
\label{Fig5.smac}
\end{figure*}

\section{Experiments}
\label{section: exp}
In this section, we evaluate our proposed CORD in a variety of cooperative multi-agent tasks including resource collection \citep{copa} in MPE\citep{mpe} and SMAC\citep{smac,refil} to empirically investigate whether CORD can enable better generalizable cooperation.
To ensure reproducibility, we include our code in the supplementary material and will make it open-source upon acceptance.

\subsection{Experimental Setting}

\textbf{Environment.} In resource collection, agents collect dispersed resources, facing invaders \& defending a home.


Multi-task SMAC utilizes variable types of agents with entity-defined states \& mask-based observations. Our experiments include two scenarios: \texttt{sz} and \texttt{MMM}. Additional two (\texttt{m} and \texttt{csz}) are available in \cref{exp_more}.

\textbf{Baselines and Ablation.}
We compare our method with \texttt{hierarchical RL methods}, COPA \citep{copa} and ALMA \citep{iqbalalma}, \texttt{CTDE with full observation methods}, Attention QMIX (AQMIX) and REFIL \citep{refil}, \texttt{communication-based method}, SOG \citep{shao2022self}, and \texttt{role-based method}, ROMA \citep{roma}.
Moreover, we introduce two ablation baselines. We fix the role distribution $\mathcal{P}(\mathbf{c})$ to be uniform at random, and we denote this baseline as MaxEnt. For the second one, we remove the intrinsic rewards from CORD, denoted as CORD w/o I.

\textbf{Training and Generalization Test.}
All methods are trained in multi-task settings with varying agent numbers: 2-4 in resource collection and 3-7 in SMAC during training.

To evaluate generalizable cooperation, we set up two types of generalization tests. 
\begin{itemize}
    \item \textit{\textbf{Generalization to unseen teams.}} The learned policy is applied to tasks with the number of agents different from training.
    \item \textit{\textbf{Generalization to unseen agents.}} 
    The learned policy is applied to control part of the team to cooperate with the remaining built-in, unseen agents in each environment.
\end{itemize}

All results are presented using the mean and standard deviation of five runs with different random seeds unless stated otherwise. More details about experimental settings and hyperparameters are available in Appendix~\ref{exp_detail}.


\subsection{Resource Collection}
\label{exp:rc}


Figure \ref{Fig4.sub.1} shows CORD significantly outperforms all baselines on training tasks, indicating well-coordinated policies from effective role assignment and a reasonable controller policy. Figures \ref{Fig4.sub.2} and \ref{Fig4.sub.3} demonstrate CORD's superior performance in generalization tests with unseen 5-agent and 6-agent teams, suggesting its role assignment's high generalizability. 

To evaluate generalization to unseen collaborators, we test 5-agent and 6-agent tasks with the setting of 1 to 4 controllable agents and 1 to 5 controllable agents, respectively, and report their average episode rewards for 5-agent and 6-agent tasks respectively in Table \ref{table1}. While all methods show decreased performance, CORD consistently outperforms baselines, demonstrating its superior generalization to unseen agents.


\begin{table}[tb!]
\footnotesize
\renewcommand{\arraystretch}{1.1}
      \caption{Average episode rewards of generalization tests for \textbf{\textit{unseen agents}} in resource collection, where we bold the highest episode rewards.}
          \label{table1}
          \centering
          \vspace{2mm}
            \begin{tabular}{|c||c|c|}
            \hline
            \diagbox{Methods}{Tasks}     & 5-agent task     &  6-agent task\\    
            \hline\hline
            AQMIX  & $115.69 \pm 3.54$  & $139.58 \pm 3.58$     \\
            \hline
            REFIL & $59.88 \pm 3.04$  & $70.99 \pm 2.79$     \\  
            \hline
            COPA & $130.02 \pm 2.24$  & $149.86 \pm 1.37$     \\
            \hline
            SOG  & $58.62\pm 28.99$  & $72.72\pm 32.56$     \\
            \hline
            ALMA & $54.86\pm 31.42$  & $65.85\pm 35.12$     \\
            \hline
            ROMA & $52.49\pm 25.58$  & $79.64\pm 34.50$     \\
            \hline\hline
            CORD     & $\mathbf{139.16} \pm 2.45$ & $\mathbf{166.334} \pm 2.65$     \\
            \hline
            \end{tabular}
\end{table}

\begin{table*}[t!]
\footnotesize
\renewcommand{\arraystretch}{1.1}
  \caption{Win rates of generalization tests for \textit{\textbf{unseen teams}} in SMAC, where we bold the highest win rate.}
  \label{table2}
  \vspace{2mm}
  \centering
  \small
  \begin{tabular}{|c||c|c|c|c|}
    \hline
    \diagbox{Methods}{Tasks}      & 2sz     &  8sz &  2MMM  & 8MMM\\
    \hline\hline
    AQMIX  & $0.161\pm 0.055$  & $0.171\pm 0.026$  & $0.187\pm 0.027$  &$0.192\pm 0.014$     \\
    \hline
    REFIL & $0.594 \pm 0.076$  & $0.434 \pm 0.011$ & $0.254 \pm 0.034$  & $0.596 \pm 0.032$    \\
    \hline
    COPA & $0.660 \pm 0.046$  & $0.370 \pm 0.214$ & $0.252 \pm 0.051$ & $0.612 \pm 0.068$     \\
    \hline
    SOG  & $0.469 \pm 0.059$  & $0.425 \pm 0.049$  & $0.147 \pm 0.135$ & $0.414 \pm 0.135$ \\
    \hline
    ALMA & $0.371 \pm 0.093$  & $0.375 \pm 0.072$  &$0.164 \pm 0.058$ & $0.617 \pm 0.062$  \\
    \hline
    ROMA & $0.063 \pm 0.024$  & $0.052 \pm 0.032$ & $0.043 \pm 0.016$ & $0.057 \pm 0.023$   \\
    \hline\hline
    CORD     & $\mathbf{0.696} \pm 0.027$ & $\mathbf{0.491} \pm 0.055$ & $\mathbf{0.288} \pm 0.085$ & $\mathbf{0.713} \pm 0.003$     \\
    \hline
  \end{tabular}
\end{table*}

\begin{table*}[t!]
  \caption{Ablation study of CORD in resource collection: training tasks, generalization tests for \textbf{\textit{unseen teams}}, and generalization tests for \textbf{\textit{unseen agents}}.}
  \label{ablation}
  \vspace{2mm}
  \centering
  \footnotesize
\renewcommand{\arraystretch}{1.1}
\begin{tabular}{|c|c|cc|cc|}
\hline
\multirow{2}{*}{\diagbox{Methods}{Tasks}} & \multirow{2}{*}{training tasks} & \multicolumn{2}{c|}{unseen teams} & \multicolumn{2}{c|}{unseen agents} \\ \cline{3-6} &  & \multicolumn{1}{c|}{5-agent task} & 6-agent task & \multicolumn{1}{c|}{5-agent task} & 6-agent task \\ \hline\hline
MaxEnt            &        $68.89 \pm 26.03$                   & \multicolumn{1}{c|}{$118.38 \pm 28.98$ }         &     $140.10 \pm 32.80$     & \multicolumn{1}{c|}{$110.01 \pm 6.36$}         &    $132.78 \pm 5.76$      \\ \hline
CORD w/o I        &      $115.48 \pm 9.71$                     & \multicolumn{1}{c|}{$176.12 \pm 12.04$}         &        $202.08 \pm 13.35$    & \multicolumn{1}{c|}{$92.99 \pm 15.72$}         &      $113.02 \pm 14.81$    \\ \hline \hline
CORD              &        $\mathbf{138.52} \pm 15.57$                    & \multicolumn{1}{c|}{$\mathbf{183.37} \pm 6.70$}         &   $\mathbf{210.53} \pm 7.71$       & \multicolumn{1}{c|}{$\mathbf{139.16} \pm 7.35$}         &      $\mathbf{166.334} \pm 7.96$    \\ \hline
\end{tabular}
\end{table*}

\subsection{SMAC}
\label{exp_smac}


Training results for 3-7sz and 3-7MMM scenarios are shown in Figures \ref{Fig5.sub.1} and \ref{Fig5.sub.2}, with CORD, COPA, and REFIL outperforming others. In the generalization test for unseen teams, 2sz, 8sz, 2MMM, and 8MMM, as shown in Table~\ref{table2}, CORD consistently surpasses baselines across all tasks.


For generalization tests with unseen agents in 5sz and 5MMM tasks, involving 1 to 4 controllable agents, Figure~\ref{Fig5.sub.3} and \ref{Fig5.sub.4} show CORD is again superior to all baselines despite observed performance degradation for all methods. 
Together with two more scenarios in SMAC: \texttt{m} and \texttt{csz}, available in \cref{exp_more}, there are 12 generalization tasks in SMAC. \textbf{CORD obtains the best win rate in 8 out of 12 tasks, while ALMA, COPA, and SOG are respectively 2/12, 1/12, and 1/12.}

In summary, despite strong training results, CTDE methods show decline in generalization tests due to inability to leverage global info for unseen teamwork. ROMA neglects agent interactions, struggling with collaboration in diverse teams. SOG's miscommunication risks faulty policies in generalization. COPA crudely manages global info, performing badly in hetergeneous environments. ALMA relies heavily on preset sub-task counts, working only in homogeneous settings (e.g., \texttt{m}). Next, we perform an ablation study to better understand CORD.



\subsection{Ablation Study}
\label{sec_ablation}

We conduct an ablation study in resource collection to investigate (1) whether a prior role distribution with the maximum entropy, denoted as MaxEnt, is enough for generalizable cooperation; (2) whether it is enough to learn only from environmental rewards, denoted as CORD w/o I, indicating the benefit of our objective function \eqref{f7}. For the ablation study in resource collection, the setting of training and generalization test is the same as in \cref{exp:rc}. 

As shown in Table~\ref{ablation}, CORD substantially outperforms both CORD w/o I and MaxEnt on the training tasks, generalization to unseen teams, and generalization to unseen agents. With our objective function converted as intrinsic rewards, CORD obtains a large performance gain over CORD w/o I. CORD is better than MaxEnt in all the settings, indicating a prior role distribution with the maximum entropy offers little guidance to role assignment and even degrades the performance. On the other hand, MaxEnt outperforms CORD w/o I in generalization to unseen agents, highlighting the potential benefit of the maximum entropy principle may help there. 


To this end, we can conclude that directly optimizing the prior role distribution may not be beneficial for generalizable cooperation. However, our objective function, i.e., optimizing the entropy with constraints, effectively improves both training and generalization across different teams with unseen agents.

\section{Conclusion and Limitation}
\label{section: conclusion}
In this paper, we propose CORD, a hierarchical MARL approach leveraging role diversity for generalizable cooperation. A high-level controller assigns roles to low-level agents, whose policies depend on these roles. We formulate the problem of generalizable role assignment as the constrained optimization of entropy and mathematically decompose the objective into two terms: causal inference in role and role heterogeneity. The two terms are further converted to intrinsic rewards and optimized end-to-end. Empirically, we evaluate CORD in a variety of cooperative multi-agent tasks. Results show CORD substantially outperforms baselines in generalization tests for unseen teams and unseen agents. Ablation studies verify the efficacy of our constrained optimization objective.  

A limitation of CORD is it periodically requires global information and assigns roles to agents during execution, which may limit the flexibility and adaptability of decentralized multi-agent systems. Yet, this represents a tradeoff between generalizability and decentralization.

\bibliography{NIPS2023/ref}

\appendix
\onecolumn
\section{Proof of Theorem \ref{t1}}
\label{section: proof}
\label{SPI_proof}

\begin{proof}
First, given the definitions of mutual information and entropy, the equation \eqref{f7} is valid. We then prove that the equation \eqref{eq:causal} is valid. As shown in Lemma \ref{l3}, the mutual information between roles and influence from other agents can be estimated using the $\DO$ operator \citep{pearl2009causal}.
\begin{lemma}\label{l3}
    Suppose that N agents satisfy the causality for any agent $i$ which is shown in the causal graph (Figure~\ref{fig2:causalgraph}), and there is a multivariate role distribution $\mathcal{P(\mathbf{c})}$ which satisfies Definition \ref{d2}, then 
\begin{align*}
        \kl[\postjt||\margjt] = \sum_{i=1}^{N} \kl[\posti || \margdoi].
\end{align*}
\end{lemma}

\begin{proof}
\begin{align*}
        \kl[\postjt || \margjt] &= \int \! \postjt \log \frac{\postjt}{\margjt} \, d\rolejt \\
                                &= \int \! \postjt \log \frac{\prod_{i=1}^{N}\posti}{\prod_{i=1}^{N}\margi} \, d\rolejt \\
                                &= \int \! \postjt \sum_{i=1}^{N} \log \frac{\posti}{\margi} \, d\rolejt \\
                                &= \sum_{i=1}^{N} \int \! \postjt \log \frac{\posti}{\margi} \, d\rolejt \\
                                & = \sum_{i=1}^{N} \int \! \prod_{i=1}^{N} \posti \log \frac{\posti}{\margi} \, d\rolejt \\
                                & = \sum_{i=1}^{N} \int \! \posti \log \frac{\posti}{\margi} \, d\rolei \\
                                & = \sum_{i=1}^{N} \kl[\posti || \margi] \\
                                & = \sum_{i=1}^{N} \kl[\posti || \sum_{\counterinflui}\mathcal{P}(\rolei|\counterinflui, \queryi)\mathcal{P}(\counterinflui|\queryi)],
\end{align*}
where $\counterinflui$ is any potential counterfactual influence vector from other agents. Based on \citet{pearl2009causal} and \citet{jaques2019social}, we can use $\DO(\influjt)$ to represent the average influence from other agents. Therefore, 
\begin{align*}
        \E_{\taujt}[\kl[\postjt || \margjt]] =  \E_{\taujt}[\sum_{i=1}^{N} \kl[\posti || \margdoi]].
\end{align*}
\end{proof}

Lastly, we demonstrate that the equation \eqref{eq:role-H} is valid by Lemma \ref{l2}, where we show that the entropy of the posterior role distribution $\postjt$ corresponds to the determinant of the role-related matrix. 
\begin{lemma}\label{l2}
 Suppose that both prior role distribution $\priorjt$ and the posterior role distribution $\postjt$ obey Gaussian distribution, $\mathbf{c} \sim \mathcal{N}(\boldsymbol{\mu}', \boldsymbol{\Sigma})$ given $\influjt$ and $\queryjt$, then the role heterogeneity can be transformed to the entropy of the posterior role distribution.
\end{lemma}
It can be formalized as:
\begin{align*}
        \Ent(\postjt) = \beta\log |A(\rolejt)| + C,
\end{align*}
where $\beta$ and C are constants.
\begin{proof}
Suppose that the $\mathbf{c}$-related matrix $A(\mathbf{c})$ satisfies Definition~\ref{d2}, then $A(\mathbf{c})$ = $\boldsymbol{\Sigma}$ according to the definition of a valid covariance matrix for a Gaussian process \citep{smola1998learning, abrahamsen1997review}, where $\boldsymbol{\Sigma}$ is the covariance matrix of the posterior role distribution $\mathcal{P}(\mathbf{c}|\boldsymbol{\bar{I}}, \mathbf{q})$, then:
\begin{align*}
        \Ent(\postjt) = \frac{N}{2}\log(2\pi e) + \frac{1}{2}\log |A(\rolejt)|.
\end{align*}

Therefore $\mathcal{H}(\mathcal{P}(\rolejt|\boldsymbol{\bar{I}}, \mathbf{q}))$ can be represented as:
\begin{align*}
    \mathcal{H}(\mathcal{P}(\rolejt|\boldsymbol{\bar{I}}, \mathbf{q})) = \beta \log |A(\mathbf{c})| + C,
\end{align*}
where $\beta$ and C are constants.
\end{proof}

Combining Lemma \ref{l3} and \ref{l2}, we can have the following equation to decompose the entropy of the prior role distribution:
\begin{align*}
    \Ent(\priorjt) &= \MI(\rolejt;\influjt|\queryjt) + \Ent(\postjt) \\
                   &= \MI(\rolejt;\influjt|\queryjt) + \beta\log |A(\rolejt)| + C \\
                   &= \E_{\taujt}[\sum_{i=1}^{N} \kl[\posti || \margdoi]] + \beta\log |A(\rolejt)| + C,
\end{align*}
which concludes the proof of Theorem \ref{t1}.
\end{proof}

\section{Pseudocode of CORD}
\label{section: code}
\label{app:algorithm}
In this section, we present the pseudocode of CORD in Algorithm \ref{code1}, which corresponds to the framework in Section \ref{frameworkSection}.

\begin{algorithm}[ht!]
\begin{algorithmic}
\STATE Initialize replay memory $D$\\
\STATE Initialize the posterior role policy network $G$ with random parameters $\delta$\\
\STATE Initialize $[Q^i]$, $Q^{tot}$ with random parameters $\theta$  \\
\STATE Initialize target parameter $\bar \theta=\theta$ \\
\STATE {\bf Input:}
$[\boldsymbol{\bar{I}}_t^i]$, $[\mathbf{q}_t^i]$ for the posterior role policy network $G$ satisfying the definition \ref{d1}\\
\STATE {\bf Output:}
$[\mathcal{P}(\mathbf{c}_i|\boldsymbol{\bar{I}}_t^i, \mathbf{q}_t^i)]$, $A(\mathbf{c})$ satisfying the definition \ref{d2}, $[Q^i]$, $Q^{tot}$\\
\FOR{episode = 1 to M} 
    \STATE Observe initial state $\mathbf{s}_0$ and observation $\mathbf{o}_0 = [O(\mathbf{s}_0, i)]^N_{i=1}$ for each agent $i$\\
    \FOR{$t$ = 1 to T} 
    \STATE With probability $\epsilon$ select a random action $\boldsymbol{a}^i_t$\\
    \STATE Otherwise $\boldsymbol{a}_t^i=\arg \max_{\boldsymbol{a}^i_t}Q^i(\boldsymbol{\tau}^i_t, \boldsymbol{a}^i_t)$ for each agent $i$\\
    \STATE Take action $\boldsymbol{a}_t$ and retrieve next observation and reward ($\mathbf{o}_{t+1}, r_t$)\\
    \STATE Store transition ($\boldsymbol{\tau}_t, \boldsymbol{a}_t, r_t, \boldsymbol{\tau}_{t+1}$) in $D$\\
    \STATE Sample a random minibatch of transitions ($\boldsymbol{\tau}, \boldsymbol{a}, r, \boldsymbol{\tau}'$) from $D$\\
    \STATE Compute $r_c$ and $r_d$ by the equation \eqref{r_c} and \eqref{r_d} with outputs $[\mathcal{P}(\mathbf{c}_i|\boldsymbol{\bar{I}}_t^i, \mathbf{q}_t^i)]$ and $A(\mathbf{c})$\\
    \STATE Set $\lambda_c$ and $\lambda_d$, compute the shaped reward $r$ by the equation \eqref{newReward} \\
    \STATE Update $\theta$ by minimizing the loss \eqref{eq2}\\
    \STATE Update target network parameters $\bar \theta=\theta$ with period $J$\\
    \ENDFOR
\ENDFOR
\end{algorithmic}
\caption{CORD}
\label{code1}
\end{algorithm}

\section{Experiment Settings and Implementation Details}
\label{section: exp details}
\label{exp_detail}
\subsection{Navigation Control in MPE}
\label{nc}
This task involves 4 agents and 4 landmarks. Each agent can take one of five actions. The objective for agents is to minimize the total distance to all targets while avoiding collisions in 200 timesteps. For generalization to unseen agents, we also use built-in agents as unseen agents and test the performance in a 4-agent task. For CORD, we use a learning rate of $3 \times 10^{-4}$. The hyperparameters of \textit{causal inference in role} $\lambda_c$ and \textit{role heterogeneity} $\lambda_d$ are fixed as 0.001 throughout the 6M training timesteps. The batch size used in the experiment is 256. The controller network of CORD contains 8 layers. Three fully connected layers are used to encode observation, action, and entity information into three 128-dimensional vectors respectively. One multi-head attention layer with 4 heads takes observation embedding vectors as input and outputs the 128-dimensional hidden vector of the global information. One fully connected layer encodes the hidden vector of the global information and observation embedding vectors and outputs the 128-dimensional vector of individual information for each agent. Another multi-head attention layer with 4 heads takes the vectors of the other agents' information as input for generating the 128-dimensional influence vectors about other agents. The MLP head layer contains 2 fully connected layers taking the individual information and influence vectors as input, generating the mean and variance of the distribution respectively. The target networks are updated after every 200 training episodes. For the environment, we use MPE~\citep{mpe} with MIT license. We implement the default configurations of AQMIX \citep{copa} and REFIL \citep{shao2022self} with the MIT license. We use the original code of COPA \citep{copa}. For ALMA \citep{iqbalalma}, we integrated the code of the resource collection environment and set the number of subtasks to one. ROMA \citep{roma} is adapted by migrating entity-based environment code into the algorithm, enabling ROMA to accommodate tasks with varying numbers of agents. Lastly, for SOG \citep{shao2022self}, we also utilize the original code directly. The environment and model are configured in Python. All models are constructed utilizing PyTorch and trained on a machine with 1 Nvidia GPU (GTX 1080 TI) and 12 Intel CPU cores.

\subsection{Resource Collection}
In resource collection \citep{copa}, agents cooperate to gather 3 different types of resources (red, green, blue) scattered in the environment. The scenario includes 3 entities: invaders, agents, and home. The goal of this task is to collect resources and bring them home while intercepting invaders to defend the home. Agents have 5 actions and observe entities within 0.2 units. Agents can only hold one resource at a time, so they must bring the holding resource home before collecting another one. Invaders periodically appear and move to the home. Episodes last 145 timesteps. In resource collection, we implement the default settings of COPA and AQMIX in \citet{copa} and we use the configuration of REFIL in \citet{shao2022self} with the MIT license. For the implementation of CORD, the details are the same in MPE (\ref{nc}) except we train all methods for 10M timesteps. The environment and model are implemented in Python. All models are built on PyTorch and are trained on a machine with 1 Nvidia GPU (GTX 1080 TI) and 12 Intel CPU Cores.

\subsection{SMAC}
In StarCraft II, for CORD, we use a learning rate of $5 \times 10^{-4}$. The hyperparameters of \textit{causal inference in role} $\lambda_c$ and \textit{role heterogeneity} $\lambda_d$ are fixed as 0.0025 throughout the 6M training timesteps for all maps. Except the above three parameters, the implementation of CORD is the same as it in MPE. For AQMIX and REFIL, we implement the default configurations for each scenario. Our implementation of CORD and COPA derives from REFIL \citep{refil} with the MIT license. The environment and model are configured in Python. All models are constructed utilizing PyTorch and trained on a system with 4 Nvidia GPUs (A100) and 224 Intel CPU cores.

\section{More Experiments on MPE and SMAC}
\label{section: more exp}
\label{exp_more}
\subsection{Navigation Control in MPE}
We further evaluate CORD on Navigation Control in MPE tasks. As shown in Table \ref{table6}, CORD achieves markedly superior performance than all baselines in terms of training performance. The enhanced performance indicates that CORD has formulated a sound controller policy for role assignment. In terms of the generalization for unseen agents, CORD substantially outperforms all baselines. This result demonstrates the role assignment of CORD also generalizes more comprehensively to unseen agents.

\begin{table}[t!]
      \caption{Average episode rewards of training and generalization tests for \textbf{\textit{unseen agents}} (4-agent task) in navigation control, where we bold the highest episode rewards.}
          \label{table6}
          \vspace{2mm}
          \centering
            \centering
  \footnotesize
\renewcommand{\arraystretch}{1.1}
            \begin{tabular}{|c||c|c|}
            \hline
            Methods     & Training performance     &  Generalization \\    
            \hline\hline
            AQMIX  & $-346.80 \pm 33.50$  & $-327.99 \pm 18.74$     \\
            \hline
            REFIL & $-128.58 \pm 23.76$  & $-350.15 \pm 3.18$     \\  
            \hline
            COPA & $-334.67 \pm 60.31$  & $-285.21 \pm 22.11$     \\
            \hline
            SOG  & $-293.77 \pm 82.98$  & $-366.49 \pm 115.98$     \\
            \hline
            ALMA & $-9491.02 \pm 36.91$  & $-6743.03 \pm 2040.30$     \\
            \hline
            ROMA & $-1722.81 \pm 1123.53$  & $-881.29 \pm 529.23$     \\
            \hline\hline
            CORD     & $\mathbf{-108.67} \pm 12.46$ & $\mathbf{-114.21} \pm 7.12$     \\
            \hline
            \end{tabular}
\end{table}

\subsection{More Maps in SMAC}

\begin{figure}[ht!]
\centering
\begin{subfigure}{0.228\textwidth}
	\includegraphics[width=1\linewidth]{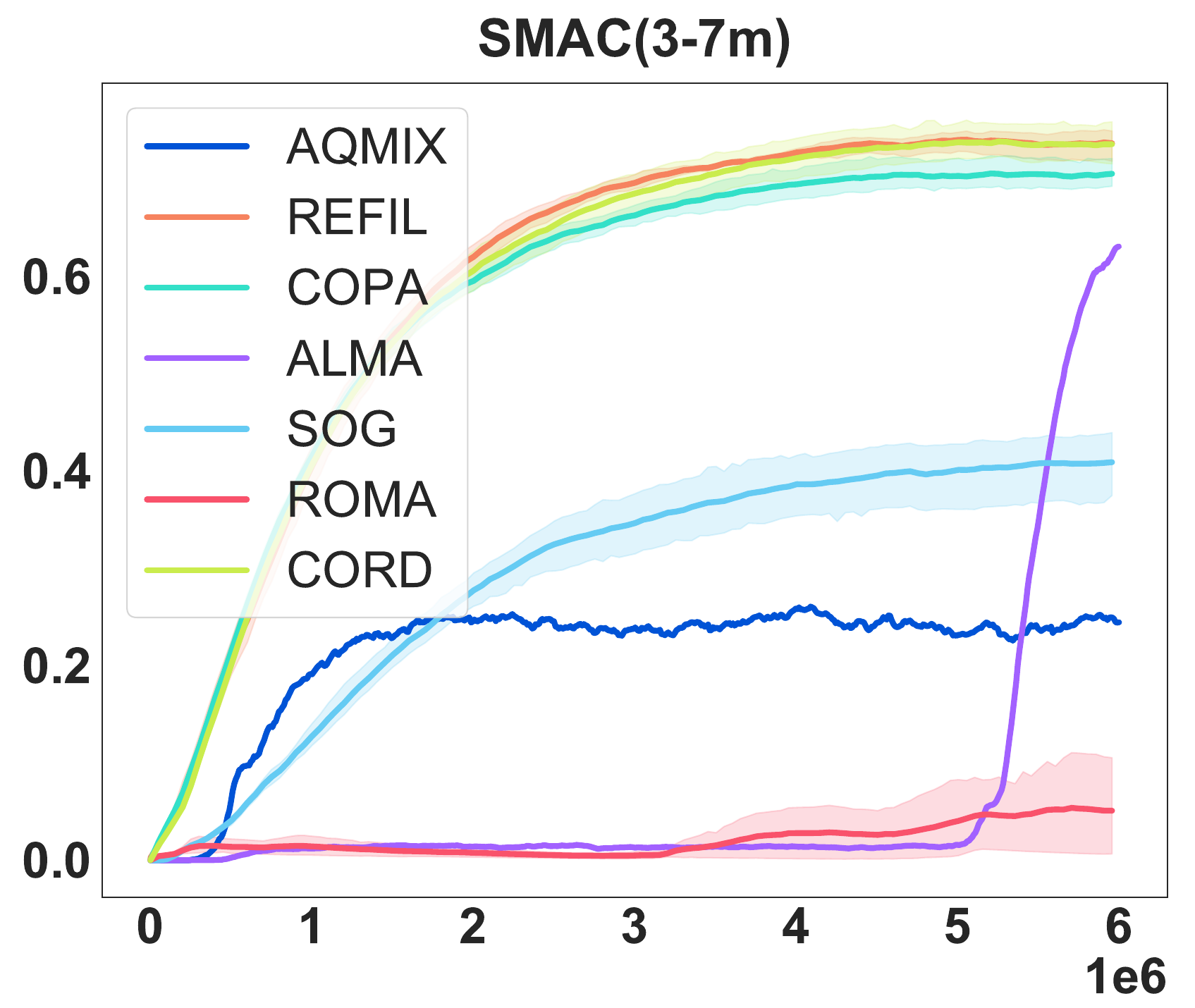}
        \caption{3-7m}
	\label{Fig7.sub.1}
\end{subfigure}
\quad
\begin{subfigure}{0.228\textwidth}
	\includegraphics[width=1\linewidth]{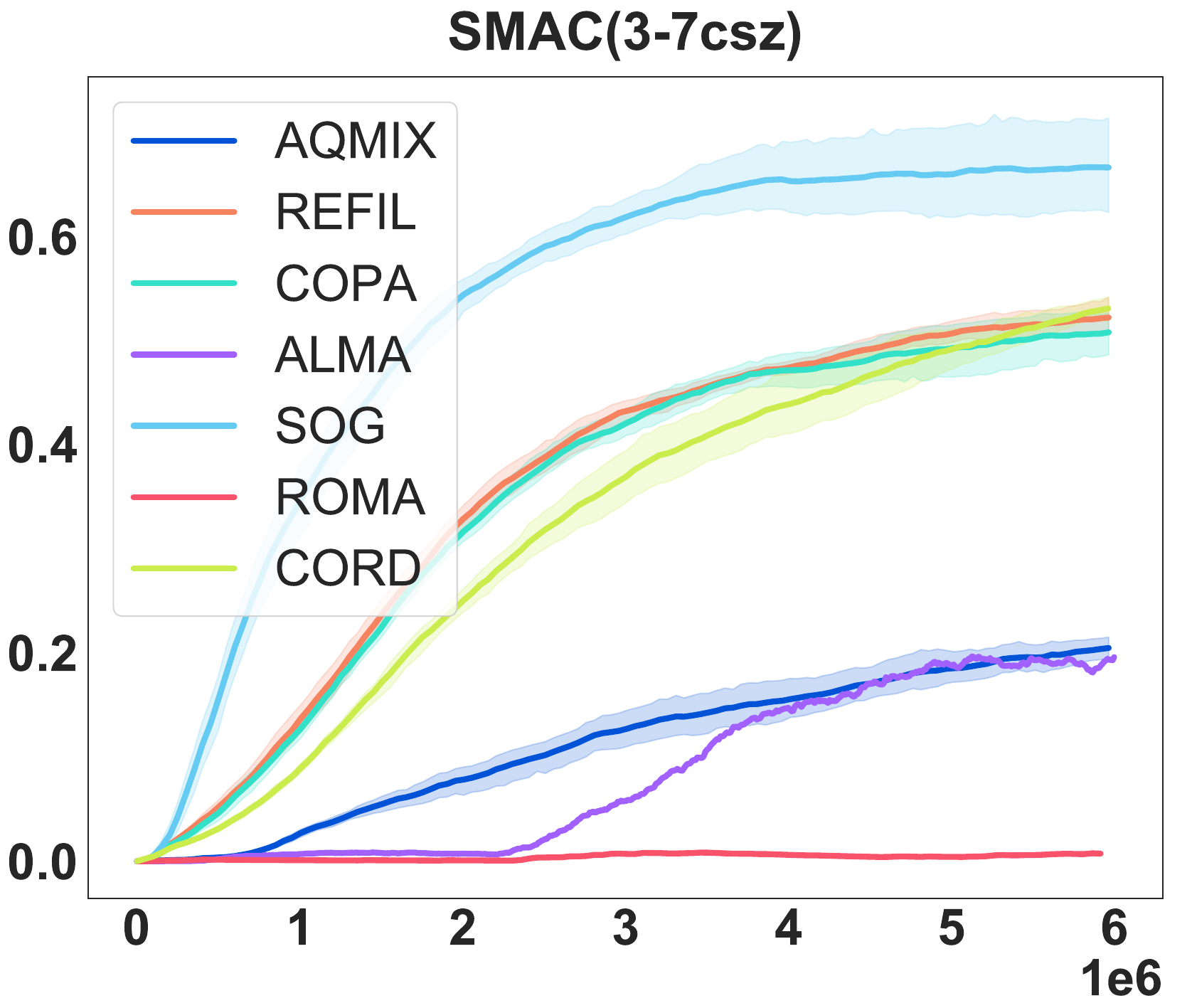}
        \caption{3-7csz}
	\label{Fig7.sub.2}
\end{subfigure}
\quad
\begin{subfigure}{0.228\textwidth}
	\includegraphics[width=1\linewidth]{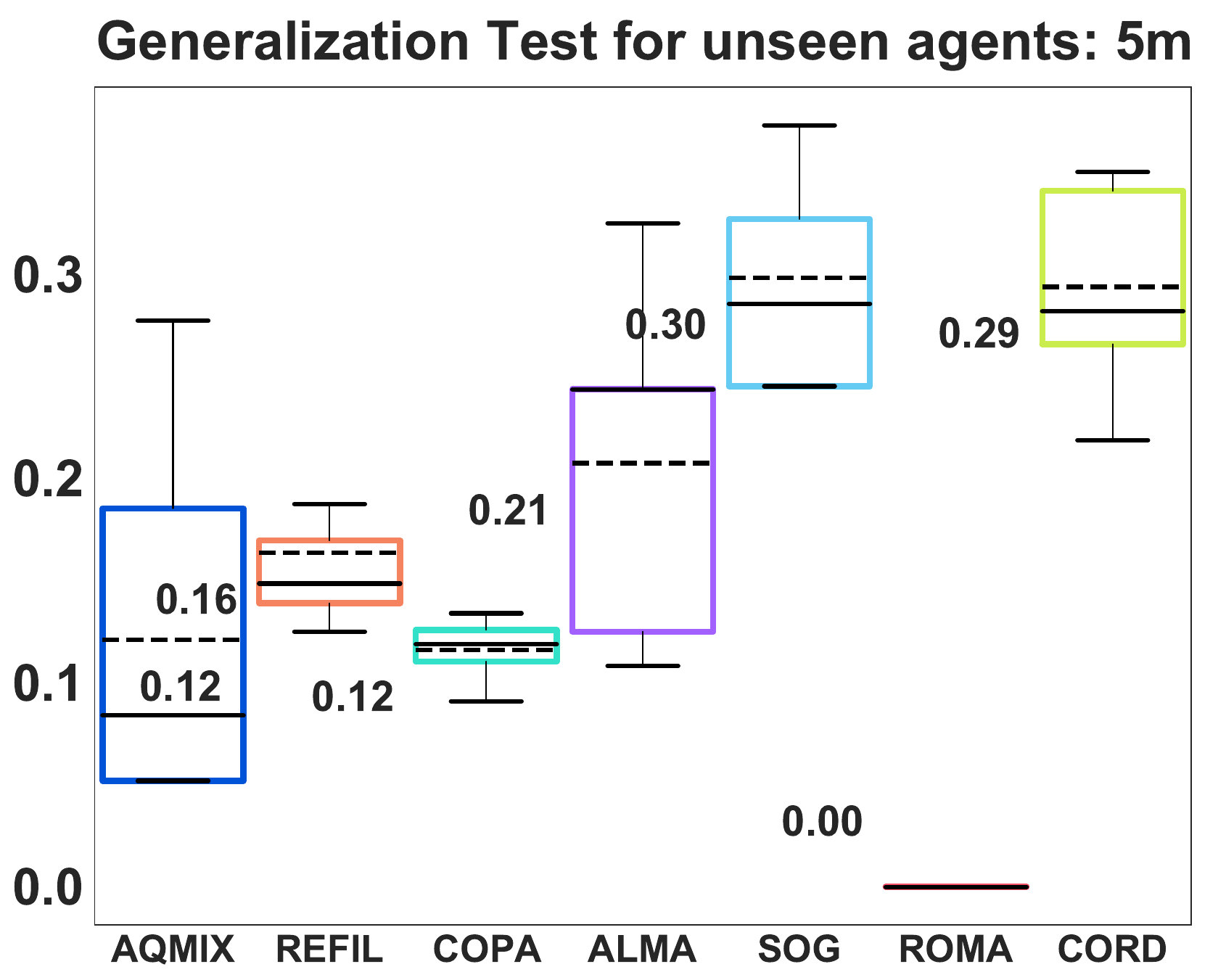}
       \captionsetup{skip=10pt}
       \caption{5m}
	\label{Fig7.sub.3}
\end{subfigure}
\quad
\begin{subfigure}{0.228\textwidth}
	\includegraphics[width=1\linewidth]{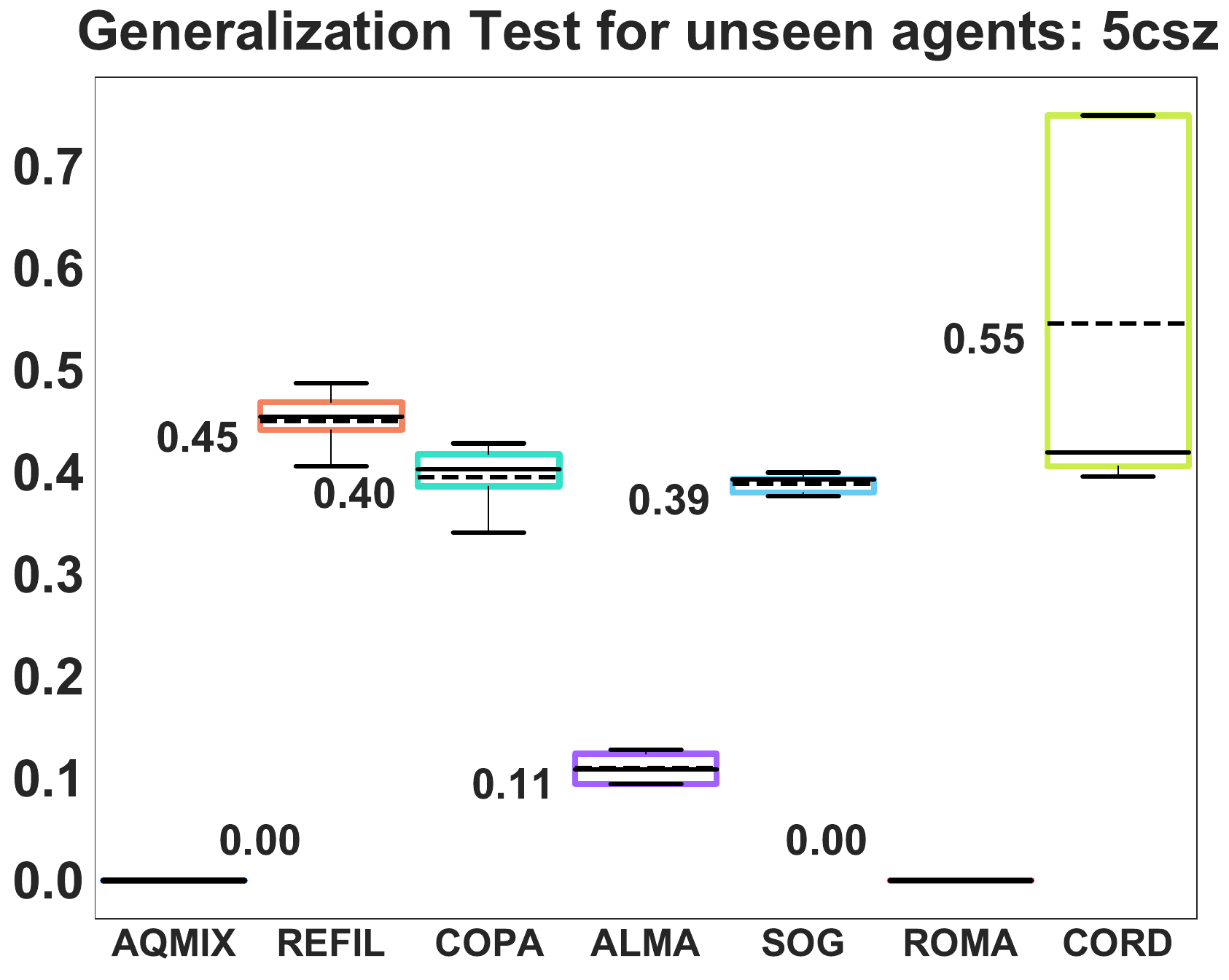}
        \captionsetup{skip=10pt}
        \caption{5csz}
	\label{Fig7.sub.4}
 \end{subfigure}
\caption{
Win rates on training and generalization to \textbf{\textit{unseen agents}} of CORD compared with all baselines on more maps in SMAC: \subref{Fig7.sub.1} learning curves in 3-7m, \subref{Fig7.sub.2} learning curves in 3-7csz, \subref{Fig7.sub.3} generalization to 5m, \subref{Fig7.sub.4} generalization to 5csz, where \subref{Fig7.sub.3} and \subref{Fig7.sub.4} are the box plot.}
\label{Fig7.moreMapsmac}
\end{figure}

\begin{table}[htb!]
  \caption{Win rates of generalization tests for \textit{\textbf{unseen teams}} on more maps in SMAC, where we bold the highest win rate.}
  \label{table7}
  \vspace{2mm}
  \centering
  \footnotesize
  \renewcommand{\arraystretch}{1.1}
  \begin{tabular}{|c||c|c|c|c|}
    \hline
    \diagbox{Methods}{Tasks}     & 2m     &  8m &  2csz  & 8csz\\
    \hline\hline
    AQMIX  & $0.481\pm 0.107$  & $0.485 \pm 0.070$  & $0.224\pm 0.060$  &$0.226\pm 0.051$     \\
    \hline
    REFIL & $0.564 \pm 0.064$  & $0.656 \pm 0.090$ & $0.520 \pm 0.039$  & $0.572 \pm 0.061$    \\
    \hline
    COPA & $0.570 \pm 0.014$  & $0.488 \pm 0.316$ & $\mathbf{0.583} \pm 0.045$ & $0.530 \pm 0.036$     \\
    \hline
    SOG  & $0.694 \pm 0.123$  & $0.674 \pm 0.063$  & $0.425 \pm 0.091$ & $0.423 \pm 0.065$ \\
    \hline
    ALMA & $\mathbf{0.809} \pm 0.021$  & $\mathbf{0.791} \pm 0.038$  &$0.256 \pm 0.050$ & $0.272 \pm 0.061$  \\
    \hline
    ROMA & $0.056 \pm 0.023$  & $0.048 \pm 0.031$ & $0.060 \pm 0.027$ & $0.057 \pm 0.031$   \\
    \hline\hline
    CORD     & $0.593 \pm 0.031$ & $0.655 \pm 0.013$ & $0.543 \pm 0.038$ & $\mathbf{0.590} \pm 0.060$     \\
    \hline
  \end{tabular}
\end{table}

We further evaluate CORD on two more SMAC maps, 3-7m and 3-7csz. As we can see in Figure \ref{Fig7.sub.1} and \ref{Fig7.sub.2}, CORD, COPA, and REFIL achieve comparable performance in training. In terms of generalization to unseen teams, as shown in Table \ref{table7}, CORD outperforms baselines on the 8csz task. Furthermore, in terms of the generalization to unseen agents, we assess 5m and 5csz tasks with the setting of 1 to 4 controllable agents. As illustrated in Figure~\ref{Fig7.sub.3} and \ref{Fig7.sub.4}, performance decay is still observed for all methods. However, CORD's performance is comparable to that of the best baseline in the 5m scenario, achieving the desired win rate of 0.29. In the 5csz scenario, CORD slightly outperforms all baselines, reaching the win rate of 0.55. Together with two maps in \cref{exp_smac}, there are a total of 12 generalization tasks in SMAC. {\textbf{CORD obtains the best win rate in 8 out of 12 tasks, while ALMA, COPA, and SOG are respectively 2/12, 1/12, and 1/12.}} 

As for the communication-based method, it is beneficial for processing information from other agents but is effective primarily in homogeneous environments like \texttt{m} maps. The misinterpretation of information in communication probably leads to agents misunderstanding each other's data and consequently producing incorrect collaborative policies. Therefore, under complex role distributions, the efficacy of communication diminishes. According to the experimental results, although SOG demonstrates more prominent training performance on the \texttt{csz} map, the generalization results indicate that SOG experiences the most significant performance degradation, exhibiting overfitting phenomena. Consequently, its generalization ability is relatively weaker compared to other models. In hierarchical RL approaches, COPA processes global information in a rudimentary manner, resulting in better performance in homogeneous environments. However, in tasks with complex role combinations, COPA is inferior to CORD, particularly in considering the impact of other agents for role assignment. For the 2csz map, the team consists of only two agents, thus having at most two unit types, with half the scenarios featuring just one type. This essentially makes the 2csz map a homogeneous environment, which is advantageous for COPA. ALMA relies significantly on the predefined number of subtasks, set to one in our experiments, making it effective only in homogeneous environments. For the \texttt{m} maps, where 100\% of scenarios involve single unit types, creating a completely homogeneous environment. ALMA, with its subtask number set to one, is highly advantageous. Given the characteristics of subtasks, ALMA requires a predefined subtask number to potentially guide diversity in role distribution. Finally, the reason for CORD's inferior performance on the \texttt{m} and \texttt{2csz} maps may be attributed to its potential issue of excessive role diversity in homogeneous environments. CORD's overemphasis on role diversity leads to performance degradation in scenarios where diverse role distribution is unnecessary.

\section*{Impact Statements}
This paper presents work whose goal is to advance the field of Machine Learning. There are many potential societal consequences of our work, none of which we feel must be specifically highlighted here.

\end{document}